\documentclass[11pt,a4]{article}
  \usepackage{color,graphicx}
  \usepackage[english]{babel}
  \usepackage{t1enc}
 \usepackage{multirow}
 \usepackage{booktabs}
\usepackage[small,bf]{caption}
\usepackage[left=2.9cm,right=2.9cm,a4paper,top=2.5cm, bottom=2.5cm]{geometry}

 \usepackage{mathptmx} \usepackage{times}
  \usepackage{amsmath}
 \usepackage{amsthm}
  \usepackage{amssymb}
\newtheoremstyle{thmm}
  {.5ex}
  {}
  {}
  {0cm}
  {\itshape}
  {.}
  {1ex}
  {}

\theoremstyle{thmm}
\newtheorem{thm}{Theorem}
\newtheorem{lem}{Lemma}
\newtheorem{alg}{Algorithm}
\newtheorem{rem}{Assumption}
\newenvironment{theorem}%
{\vspace{.1cm}\begin{flushright}\begin{minipage}{14.5cm}\begin{thm}}%
{\end{thm}\end{minipage}\end{flushright}\vspace{.1cm}}%

\newenvironment{lemma}%
{\vspace{.1cm}\begin{flushright}\begin{minipage}{14.5cm}\begin{lem}}%
{\end{lem}\end{minipage}\end{flushright}\vspace{.1cm}}%

\newenvironment{algorithm}%
{\vspace{.1cm}\begin{flushright}\begin{minipage}{14.5cm}\begin{alg}}%
{\end{alg}\end{minipage}\end{flushright}\vspace{.1cm}}%

\newenvironment{remark}%
{\begin{flushright}\begin{minipage}{14.5cm}\begin{rem}}%
{\end{rem}\end{minipage}\end{flushright}}%

 \usepackage{natbib}
 \usepackage{bm}
  
\newcommand\E {E}
\newcommand{\vek}[1]{\mathbf{#1}}
\newcommand{\mc}[1]{\mathrm{#1}}
\newcommand{\vekg}[1]{\boldsymbol{#1}}
\newcommand{\ssh}{P}
\newcommand{\cov}{\mathrm{Cov}}
\newcommand{\argmin}{\mathrm{argmin}\,}
\newcommand{\cor}{\mathrm{Cor}}
\newcommand{\var}{\mathrm{Var}}

\def\d{\mathrm{d}}
\def\e{\mathrm{e}}
\def\1{\mathrm{1}}

\title{\textbf{Independent screening for single-index hazard rate  \\ models with ultra-high dimensional features}}
\author{Anders Gorst-Rasmussen \thanks{Department of Mathematical Sciences, Aalborg University, Denmark. Email: gorst@math.aau.dk}
\and
Thomas H. Scheike \thanks{Department of Biostatistics, University of Copenhagen, Denmark}}

\begin{document}

\maketitle

\renewcommand{\abstractname}{Summary} 
\begin{abstract}
  \noindent In data sets with many more features than observations, independent
  screening based on all univariate regression models leads to a
  computationally convenient variable selection method. Recent efforts
  have shown that in the case of generalized linear models,
  independent screening may suffice to capture all relevant features
  with high probability, even in ultra-high dimension. It is unclear
  whether this formal sure screening property is attainable when the
  response is a right-censored survival time. We propose a
  computationally very efficient independent screening method for
  survival data which can be viewed as the natural survival equivalent
  of correlation screening. We state conditions under which the method
  admits the sure screening property within a general class of
  single-index hazard rate models with ultra-high dimensional
  features. An iterative variant is also described which combines
  screening with penalized regression in order to handle more complex
  feature covariance structures. The methods are evaluated through
  simulation studies and through application to a real gene expression
  data~set.
\end{abstract}


\section{Introduction}
\label{sec:intro}
With the increasing proliferation of  biomarker
studies, there is a need for efficient methods for relating a
survival time response to a large number of features. In typical
genetic microarray studies, the sample size $n$ is measured in
hundreds whereas the number of features $p$ per sample can be in
excess of millions.  Sparse regression techniques such as lasso
\citep{tibshirani97:_lasso_cox} and SCAD \citep{fan01:_variab} have
proved useful for dealing with such high-dimensional features but
their usefulness diminishes when $p$ becomes extremely large compared
to $n$. The notion of NP-dimensionality
\citep{fan09:_non_concav_penal_likel_np_dimen} has been conceived to
describe such ultra-high dimensional settings which are formally
analyzed in an asymptotic regime where $p$ grows at a non-polynomial
rate with~$n$. Despite recent progress \citep{bradic11:_penal},
theoretical knowledge about sparse regression techniques under
NP-dimensionality is still in its infancy. Moreover, NP-dimensionality
poses substantial computational challenges. When for example pairwise
interactions among gene expressions in a genetic microarray study are
of interest, the dimension of the feature space will trouble even the
most efficient algorithms for fitting sparse regression models.  A
popular ad hoc solution is to simply pretend that feature correlations
are ignorable and resort to computationally swift univariate
regression methods; so-called independent screening methods.

In an important paper, \cite{fan08:_sure} laid the formal foundation
for using independent screening to distinguish `relevant' features
from `irrelevant' ones. For the linear regression model they showed
that, when the design is close to orthogonal, a superset of the true
set of nonzero regression coefficients can be estimated consistently
by simple hard-thresholding of feature-response correlations. This
sure independent screening (SIS) property of correlation screening is
a rather trivial one, if not for the fact that it holds true in the
asymptotic regime of NP-dimensionality. Thus, when the feature
covariance structure is sufficiently simple, SIS methods can overcome
the noise accumulation in extremely high dimension. In order to
accommodate more complex feature covariance structures
\cite{fan08:_sure} and \cite{fan09:_ultrah_dimen_featur_selec}
developed heuristic, iterated methods combining independent screening
with forward selection techniques. Recently, \cite{fan10:_sure_np}
extended the formal basis for SIS to generalized linear models.

In biomedical applications, the response of interest is often a
right-censored survival time, making the study of screening methods
for survival data an important one. \cite{fan10:_borrow_stren}
investigated SIS methods for the Cox proportional hazards model based
on ranking features according to the univariate partial log-likelihood
but gave no formal
justification. \cite{tibshirani09:_univar_shrin_cox_model_high_dimen_data}
suggested soft-thresholding of univariate Cox score statistics with
some theoretical justification but under strong assumptions. Indeed,
independent screening methods for survival data are apt to be
difficult to justify theoretically due to the presence of censoring
which can confound marginal associations between the response and the
features. Recent work by \cite{zhao10:_princ_sure_indep_screen_cox}
contains ideas which indicate that independent screening based on the
Cox model may have the SIS property in the absence of censoring.

In the present paper, we depart from the standard approach of studying
SIS as a rather specific type of model misspecification in which the
univariate versions of a particular regression model are used to infer
the structure of the joint version of the same particular regression
model. Instead, we propose a survival variant of independent screening
based on a model-free statistic which we call the `Feature Aberration
at Survival Times' (FAST) statistic. The FAST statistic is a simple
linear statistic which aggregates across survival times the aberration
of each feature relative to its time-varying average. Independent
screening based on this statistic can be regarded as a natural
survival equivalent of correlation screening.  We study the SIS
property of FAST screening in ultra-high dimension for a general class
of single-index hazard rate regression models in which the risk of an
event depends on the features through some linear functional. A key
aim has been to derive simple and operational sufficient conditions
for the SIS property to hold. Accordingly, our main result states that
the FAST statistic has the SIS property in an ultra-high dimensional
setting under covariance assumptions as in
\cite{fan09:_ultrah_dimen_featur_selec}, provided that censoring is
essentially random and that features satisfy a technical condition
which holds when they follow an elliptically contoured distribution.
Utilizing the fact that the FAST statistic is related to the
univariate regression coefficients in the semiparametric additive
hazards model (\cite{lin94:_semip_analy}; \cite{mckeague94}), we
develop methods for iterated SIS. The techniques are evaluated in a
simulation study where we also compare with screening methods for the
Cox model \citep{fan10:_borrow_stren}. Finally, an application to a
real genetic microarray data set is presented.

\section{The FAST statistic and its motivation}
\label{sec:classical}
Let $T$ be a survival time which is subject to right-censoring by some
random variable $C$.  Denote by $N(t):=\1(T\land C \leq t \land T \leq
C)$ the counting process which counts events up to time $t$, let
$Y(t):=\1(T \land C \geq t)$ be the at-risk process, and let $\vek{Z}
\in \mathbb{R}^p$ denote a random vector of explanatory variables or
features. It is assumed throughout that $\vek{Z}$ has finite variance
and is standardized, i.e.~centered and with a covariance matrix
$\vekg{\Sigma}$ with unit diagonal. We observe $n$ independent and
identically distributed (i.i.d.) replicates of $\{(N_i,Y_i,\vek{Z}_i)\,:\,0
\leq t \leq \tau\}$ for $i=1,\ldots,n$ where $[0,\tau]$ is the
observation time window.

Define the `Feature Aberration at Survival Times' (FAST) statistic as
follows:
\begin{equation}
\label{eq:dev-at-event-defin}
  \vek{d} := n^{-1} \int_0^\tau \sum_{i=1}^n \{\vek{Z}_i-\bar{\vek{Z}}(t)\} \d N_i(t); 
\end{equation}
where $\bar{\vek{Z}}$ is the at-risk-average of the $\vek{Z}_i$s, 
\begin{displaymath}
  \bar{\vek{Z}}(t):=\frac{\sum_{i=1}^n \vek{Z}_iY_i(t)}{\sum_{i=1}^n Y_i(t)}. 
\end{displaymath}
Components of the FAST statistic define basic measures of the
marginal association between each feature and survival. In the following,
we provide two motivations for using the FAST statistic for screening
purposes. The first, being model-based, is perhaps the most intuitive
-- the second shows that, even in a model-free setting, the FAST
statistic may provide valuable information about marginal
associations.

\subsection{A model-based interpretation of the FAST statistic}
\label{sec:mod-based-inter}
Assume in this section that the $T_i$s have hazard
functions of the form
\begin{equation}
\label{eq:addriskmod}
 \lambda_j(t)= \lambda_0(t)+\vek{Z}_j^\top \vekg{\alpha}^0; \qquad j=1,2,\ldots,n;
\end{equation}
with $\lambda_0$ an unspecified baseline hazard rate and
$\vekg{\alpha}^0 \in \mathbb{R}^p$ a vector of regression
coefficients. This is the so-called semiparametric additive hazards
model (\cite{lin94:_semip_analy}; \cite{mckeague94}), henceforth
simply the Lin-Ying model.  The Lin-Ying model corresponds to assuming
for each $N_j$ an intensity function of the form
$Y_j(t)\{\lambda_0(t)+\vek{Z}_j^\top \vekg{\alpha}^0 \}$.  From the
Doob-Meyer decomposition $\d N_j(t)= \d M_j(t) +
Y_j(t)\{\lambda_0(t)+\vek{Z}_j^\top \vekg\alpha^0 \} \d t$ with $M_j$
a martingale, it is easily verified that
\begin{equation}
\label{eq:decomp-of-ahaz} 
  \sum_{i=1}^n\{\vek{Z}_i-\bar{\vek{Z}}(t)\}\d N_i(t) = \Big[\sum_{i=1}^n\{\vek{Z}_i-\bar{\vek{Z}}(t)\}^{\otimes 2}Y_i(t) \d t\Big]\vekg{\alpha}^0  + \sum_{i=1}^n\{\vek{Z}_i-\bar{\vek{Z}}(t)\} \d M_i(t), \quad t \in [0,\tau].
\end{equation}
This suggests that $\vekg\alpha^0$ is estimable as the solution to the
$p \times p$ linear system of equations
\begin{equation}
\label{eq:linahaz}
  \vek{d}=\vek{D}\vekg{\alpha};
\end{equation}
where
\begin{equation}
 \label{eq:def-of-smalldn-bigdn}
  \vek{d}:=n^{-1}\sum_{i=1}^n \int_0^\tau \{ \vek{Z}_i-\bar{\vek{Z}}(t)\} \d N_i(t), \quad \textrm{and  }\vek{D}:=n^{-1}\sum_{i=1}^n \int_0^\tau Y_i(t)\{ \vek{Z}_i-\bar{\vek{Z}}(t)\}^{\otimes 2} \d t.
\end{equation}
Suppose $\hat{\vekg\alpha}$ solves  \eqref{eq:linahaz}.
Standard martingale arguments \citep{lin94:_semip_analy}
imply root $n$ consistency of~$\hat{\vekg\alpha}$,
\begin{equation}
\label{eq:lin-ying-rootn-consistency}
  \sqrt{n}(\hat{\vekg\alpha}-\vekg\alpha^0) \stackrel{d}{\to} \mathrm{N}(0,\vek{D}^{-1} \vek{B} \vek{D}^{-1}), \quad \textrm{ where } \vek{B}=n^{-1}\sum_{i=1}^n\int_0^\tau \{ \vek{Z}_i-\bar{\vek{Z}}(t)\}^{\otimes 2} \d N_i(t).
\end{equation}
For now, simply observe that the left-hand side of \eqref{eq:linahaz}
is exactly the FAST statistic; whereas $d_{j}D_{jj}^{-1}$ for
$j=1,2,\ldots,p$ estimate the regression coefficients in the
corresponding $p$ univariate Lin-Ying models. Hence we can
interpret $\vek{d}$ as a (scaled) estimator of the univariate regression
coefficients in a working Lin-Ying~model.

A nice heuristic interpretation of $\vek{d}$ results from the
pointwise signal/error decomposition~\eqref{eq:decomp-of-ahaz} which
is essentially a reformulated linear regression model $\vek{X}^\top
\vek{X} \vekg{\alpha}^0 + \vek{X}^\top \vekg{\varepsilon}=\vek{X}^\top
\vek{y}$ with `responses' $y_j:=\d N_j(t)$ and `explanatory variables'
$\vek{X}_j:=\{\vek{Z}_j-\bar{\vek{Z}}(t)\}Y_j(t)$. The FAST statistic is given by the time average of $\E\{\vek{X}^\top
\vek{y}\}$ and may accordingly be viewed as a survival
equivalent of the usual predictor-response correlations.

\subsection{A model-free interpretation of the FAST statistic}
For a feature to be judged (marginally) associated with survival in
any reasonable interpretation of survival data, one would first
require that the feature is correlated with the probability of
experiencing an event -- second, that this correlation persists
throughout the time window. The FAST statistic can be shown to reflect
these two requirements when the censoring mechanism is sufficiently
simple.

Specifically, assume administrative censoring at time $\tau$ (so that
$C_1 \equiv \tau$). Set $V(t):=\mathrm{Var}\{F(t|\vek{Z}_1)\}^{1/2}$ where
$F(t| \vek{Z}_1):=\ssh(T_1 \leq t|\vek{Z}_1)$ denotes the conditional probability of death
before time~$t$.  For each $j$, denote by
$\delta_j$ the population version of $d_{j}$ (the in
probability limit of $d_{j}$ when $n \to \infty$). Then
\begin{align*}
  \delta_j &= \E \Big(\Big[Z_{1j}-\frac{\E\{Z_{1j}Y_1(t)\} }{\E\{Y_1(t)\}} \Big]\1(T_1 \leq t \land \tau)\Big) \\
  &=\E\{Z_{1j} F(\tau|\vek{Z}_1)\} - \int_0^\tau \frac{\E\{Z_{1j}Y_1(t)\} }{\E\{Y_1(t)\}} \E\{\d F(t|\vek{Z}_1)\}\\
&= V(\tau)\cor\{Z_{1j},F(\tau|\vek{Z}_1)\}+\int_0^\tau \cor\{Z_{1j},F(t|\vek{Z}_1)\}  \frac{V(t)}{\E\{Y_1(t)\}} \E\{\d F(t|\vek{Z}_1)\}.
\end{align*}
We can make the following observations:
\begin{enumerate}
\item[(i)] If $\cor\{Z_{1j},F(t|\vek{Z}_1)\}$ has
constant sign throughout $[0,\tau]$, then $|\delta_{j}|
\geq |V(\tau)\cor\{Z_{1j},F(\tau|\vek{Z}_1)\}|$.
\item[(ii)] Conversely, if $\cor\{Z_{1j},F(t|\vek{Z}_1)\}$ changes sign, so
  that the the direction of association with $F(t|\vek{Z}_1)$ is not persistent
  throughout $[0,\tau]$, then this will lead to a smaller value of
  $|\delta_{j}|$ compared to (i).
\item[(iii)]Lastly, if $\cor\{Z_{1j},F(t|\vek{Z}_1)\} \equiv 0$ then $\delta_{j}=0$.
\end{enumerate}
In other words, the sample version $d_{j}$ estimates a time-averaged summary
of the correlation function $t \mapsto \cor\{Z_{1j},F(t|\vek{Z}_1)\}$
which takes into account both  magnitude and persistent behavior 
throughout $[0,\tau]$.  This indicates that the FAST statistic is
relevant for judging marginal association of features with survival
beyond the model-specific setting of
Section~\ref{sec:mod-based-inter}.

\section{Independent screening with the FAST statistic}
In this section, we extend the heuristic arguments of the
previous section and provide theoretical justification for using the
FAST statistic to screen for relevant features when the data-generating model belongs to a  class of single-index hazard
rate regression models.

\subsection{The general case of single-index hazard rate models}
In the notation of Section \ref{sec:classical}, we assume survival times $T_j$ to have hazard rate functions of single-index form:
\begin{equation}
\label{eq:model-for-haz}
  \lambda_j(t) = \lambda(t,\vek{Z}_j^\top \vekg{\alpha}^0), \quad j=1,2,\ldots,n.
\end{equation}
Here $\lambda \colon [0,\infty) \times \mathbb{R} \to [0,\infty)$ is a
continuous function, $\vek{Z}_1,\ldots,\vek{Z}_n$ are random vectors
in $\mathbb{R}^{p_n}$, $\vekg \alpha^0 \in \mathbb{R}^{p_n}$ is a
vector of regression coefficients, and $\vek{Z}_j^\top
\vekg{\alpha}^0$ defines a risk score. We subscript $p$ by $n$ to
indicate that the dimension of the feature space can grow with the
sample size. Censoring will always be assumed at least independent so
that $C_j$ is independent of $T_j$ conditionally on $\vek{Z}_j$. We
impose the following assumption on the hazard `link
function'~$\lambda$:
\begin{remark}
\label{assumption:monotonicity}
  The survival function $\exp\{-\int_0^t \lambda(s,\,\cdot\,) \d s\}$ is strictly monotonic for each $t \geq 0$.
\end{remark}
\noindent Requiring the survival function to depend monotonically on $\vek{Z}_j^\top
\vekg{\alpha}^0$ is natural in order to enable interpretation of the
components of $\vekg{\alpha}^0$ as indicative of positive or negative
association with survival.  Note that it suffices that
$\lambda(t,\,\cdot\,)$ is strictly monotonic for each $t \geq
0$. Assumption \ref{assumption:monotonicity} holds for a range of
popular survival regression models. For example,
$\lambda(t,x):=\lambda_0(t)+x$ with $\lambda_0$ some baseline hazard
yields the Lin-Ying model~\eqref{eq:addriskmod};
$\lambda(t,x):=\lambda_0(t)\e^x$ is a Cox model; and
$\lambda(t,x):=\e^{x}\lambda_0(t\e^{x})$ is an accelerated failure
time model.

Denote by $\vekg\delta$ the population version of the FAST statistic
under the model~\eqref{eq:model-for-haz} which, by the Doob-Meyer
decomposition $\d N_1(t)= \d M_1(t) + Y_1(t)\lambda(t,\vek{Z}_1^\top
\vekg\alpha^0)\d t$ with $M_1$ a martingale, takes the form
\begin{equation}
\label{eq:cp-decom}
  \vekg{\delta} = \E\Big[\int_0^\tau \{\vek{Z}_1-\vek{e}(t)\}Y_1(t) \lambda(t,\vek{Z}_1^\top \vekg\alpha^0) \d t\Big]; \qquad \textrm{where }   \vek{e}(t):= \frac{\E\{\vek{Z}_1 Y_1(t)\}}{\E\{Y_1(t)\}}.
\end{equation}
Our proposed FAST screening procedure is as follows: given some (data-dependent) threshold $\gamma_n>0$,
\begin{enumerate}
\item[(i).] calculate the FAST statistic $\vek{d}$ from the available data and
\item[(ii).] declare the `relevant
  features' to be the set  $\{1 \leq j \leq p_n\,:\, |d_{j}|>\gamma_n\}$.
\end{enumerate}
By the arguments in Section \ref{sec:classical}, this procedure
defines a natural survival equivalent of correlation screening. Define the following sets of features:
\begin{align*}
  \widehat{\mc{M}}_{d}^n&:= \{1 \leq j \leq p_n \,:\, |d_{j}| >  \gamma_n \}, \\
  \mc{M}^n &:= \{1 \leq j \leq p_n\,:\, \alpha_j^0 \neq 0\}, \\
      \mc{M}_{\delta}^n&:= \{1 \leq j \leq p_n \,:\, \delta_{j} \neq 0\}.
\end{align*}
The problem of establishing the SIS property of FAST screening amounts
to determining when $\mc{M}^n \subseteq
\widehat{\mc{M}}_\d^n$ holds with large probability for
 large $n$. This translates into two questions: first,
when do we have $\mc{M}_\delta^n \subseteq
\widehat{\mc{M}}_d^n$; second, when do we have $\mc{M}^n
\subseteq \mc{M}_\delta^n$?  The first question is essentially
model-independent and requires establishing an exponential bound for
$n^{1/2}|d_{j}-\delta_{j}|$ as $n\to \infty$. The second question is
strongly model-dependent and is answered by manipulating expectations
under the single-index model~(\ref{eq:model-for-haz}).

We state the main results here and relegate proofs to the appendix where we also state various regularity conditions. The
following principal assumptions, however, deserve separate attention:
\begin{remark}
\label{assumption:lr}
  There exists $\vek{c} \in \mathbb{R}^{p_n}$ such that $\E(\vek{Z}_1|\vek{Z}_1^\top \vekg{\alpha}^0)=\vek{c}\vek{Z}_1^\top \vekg\alpha^0$.
\end{remark}
\begin{remark}
\label{assumption:ran-cens}
   The censoring time $C_1$ depends on $T_1,\vek{Z}_1$ only through $Z_{1j}$, $j \notin \mc{M}^n$.
\end{remark}
\begin{remark}
\label{assumption:po}
  $Z_{1j}$, $j \in \mc{M}^n$
  is independent of $Z_{1j}$, $j \notin \mc{M}^n$.
\end{remark}
Assumption \ref{assumption:lr} is a `linear regression' property which holds true for Gaussian features and, more
generally, for features following an elliptically contoured
distribution \citep{hardin82}. In view of  \cite{hall93}
which states that most low dimensional projections of high dimensional
features are close to linear, Assumption~\ref{assumption:lr} may not
be unreasonable a priori even for general feature
distributions when $p_n$ is large.

Assumption \ref{assumption:ran-cens} restricts the censoring mechanism
to be partially random in the sense of depending only on irrelevant
features. As we will discuss in detail below, such rather strong
restrictions on the censoring distribution seem necessary for
obtaining the SIS property; Assumption \ref{assumption:ran-cens} is
both general and convenient.

Assumption \ref{assumption:po} is the partial orthogonality condition
also used by \cite{fan10:_sure_np}. Under this assumption and Assumption \ref{assumption:ran-cens}, it follows
from~\eqref{eq:cp-decom} that $\delta_{j}=0$ whenever $j \notin
\mc{M}^n$, implying $\mc{M}_{\delta}^n \subseteq
\mc{M}^n$. Provided that we also have $\delta_{j} \neq 0$ for $j \in
\mc{M}^n$ (that is, $\mc{M}^n \subseteq
\mc{M}_\mathrm{pre}^n$), there exists a threshold $\zeta_n>0$ such
that
\begin{displaymath}
  \min_{j \in \mc{M}^n}|\delta_{j}| \geq \zeta_n \qquad  \max_{j \notin \mc{M}^n}|\delta_{j}|=0.
\end{displaymath}
Consequently, Assumptions
\ref{assumption:ran-cens}-\ref{assumption:po} are needed to enable
consistent model selection via independent screening. Although model
selection consistency is not essential in order to capture just some
superset of the relevant features via independent screening, it is pertinent in order to limit
the size of such a superset.

The following theorem on  FAST screening (FAST-SIS) is our
main theoretical result. It states that the screening property
$\mc{M}^n \subseteq \widehat{\mc{M}}_d^n$ may hold with
large probability even when $p_n$ grows exponentially fast in a
certain power of $n$ which depends on the tail behavior of
features. The covariance condition in the theorem is analogous to that of \cite{fan10:_sure_np} for SIS in generalized linear models
with Gaussian features.
\begin{theorem}
\label{thm:mainthm}
Suppose that Assumptions \ref{assumption:monotonicity}-\ref{assumption:ran-cens} hold alongside the regularity conditions
of the appendix and that $\ssh(|Z_{1j}|>s) \leq l_0\exp(-l_1 s^\eta)$
for some positive constants $l_0,l_1,\eta$ and sufficiently large $s$.  Suppose
moreover that  for some $c_1>0$ and $\kappa<1/2$, 
    \begin{equation}
      \label{eq:covariance-in-mainthm}
      |\cov[Z_{1j},\vek{Z}_1^\top \vekg{\alpha}^0\}]| \geq c_1n^{-\kappa}, \quad j \in \mc{M}^n.
    \end{equation}
    Then $\mc{M}^n \subseteq \mc{M}^n_\delta$. Suppose in
    addition that $\gamma_n=c_2 n^{-\kappa}$ for some constant $0<c_2
    \leq c_1/2$ and that $\log
    p_n=o\{n^{(1-2\kappa)\eta/(\eta+2)}\}$. Then the SIS property
    holds, $\ssh(\mc{M}^n \subseteq \widehat{\mc{M}}_d^n)
    \to 1$ when $n \to \infty$.
\end{theorem}
Observe that with bounded features, we may take $\eta=\infty$ and
handle dimension of order $\log p_n=o(n^{1-2\kappa})$.  

We may dispense with Assumption 2 on the feature distribution by
revising~\eqref{eq:covariance-in-mainthm}. By
Lemma~\ref{lemma:cum-haz-decom} in the appendix, taking
$\tilde{e}_j(t):= \E\{Z_{1j} \ssh(T_1 \geq t|\vek{Z}_1)\}/\E\{\ssh(T_1 \geq
t|\vek{Z}_1)\}$, it holds generally under Assumption \ref{assumption:ran-cens}~that
\begin{displaymath}
  \delta_j=  \E\{\tilde{e}_j(T_1 \land C_1 \land \tau)\}, \quad j \in \mc{M}^n.
\end{displaymath}
Accordingly, if we replace \eqref{eq:covariance-in-mainthm} with the
assumption that $\E|Z_{1j} \ssh(T_1 \geq t|\vek{Z}_1)| \geq c_1n^{-\kappa}$
uniformly in $t$ for $j \in \mc{M}^n$, the conclusions of Theorem
\ref{thm:mainthm} still hold. In other words, we can generally expect FAST-SIS to
 detect features which are `correlated with the chance of
survival', much in line with 
Section~\ref{sec:classical}. While this is valuable structural
insight, the covariance assumption \eqref{eq:covariance-in-mainthm} seems
a more operational condition.

Assumption \ref{assumption:ran-cens} is crucial to the proof of
Theorem~\ref{thm:mainthm} and to the general idea of translating a
model-based feature selection problem into a problem of
hard-thresholding $\vekg{\delta}$. A weaker assumption is not possible
in general. For example, suppose that only Assumption~\ref{assumption:lr}
holds and that the censoring time also follows some single-index model of the form
\eqref{eq:model-for-haz} with regression coefficients
$\vekg\beta^0$. Applying Lemma~2.1~of \cite{cheng94:_adjus} to~\eqref{eq:cp-decom}, there exists finite constants $\zeta_1,\zeta_2$
(depending on $n$) such that
\begin{equation}
\label{eq:effect-of-cens}
  \vekg\delta = \vekg\Sigma(\zeta_1 \vekg\alpha^0+\zeta_2\vekg\beta^0).
\end{equation}
It follows that unrestricted censoring will generally confound the relationship
between $\vekg\delta$ and $\vekg\Sigma\vekg\alpha^0$,
hence~$\vekg\alpha^0$. The precise impact of such unrestricted censoring seems difficult to discern,
although \eqref{eq:effect-of-cens} suggests that FAST-SIS may still be
able to capture the underlying model (unless $\zeta_1
\vekg\alpha^0+\zeta_2\vekg\beta^0$ is particularly ill-behaved). We
will have more to say about unrestricted censoring in the next
section.

Theorem~\ref{thm:mainthm} shows that FAST-SIS can consistently
capture a superset of the relevant features. A priori, this superset
can be quite large; indeed, `perfect' screening would result by simply
including all features. For FAST-SIS to be useful, it must
substantially reduce feature space dimension.  Below we
state a survival analogue of Theorem~5 in \cite{fan10:_sure_np}, providing
an asymptotic rate on the FAST-SIS model size.
\begin{theorem}
\label{thm:bound-on-sel-var}
Suppose that Assumptions \ref{assumption:monotonicity}-\ref{assumption:po} hold alongside
the regularity conditions of the appendix and that $\ssh(|Z_{1j}|>s)
\leq l_0\exp(-l_1 s^\eta)$ for positive constants $l_0,l_1,\eta$ and sufficiently
large $s$.  If $\gamma_n=c_4n^{-2\kappa}$ for some $\kappa< 1/2$ and
$c_4>0$, there exists a positive constant $c_5$ such~that
  \begin{displaymath}
    \ssh[|\widehat{\mc{M}}^n_d| \leq O\{n^{2\kappa}\lambda_\mathrm{max}(\vekg{\Sigma})\}] \geq 1-O(p_n\exp\{-c_5n^{(1-2\kappa)\eta/(\eta+2)}\});
  \end{displaymath}
  where $\lambda_{\mathrm{max}}(\Sigma)$ denotes the maximal
  eigenvalue of the covariance matrix $\vekg{\Sigma}$ of the feature
  distribution.
\end{theorem}
Informally, the theorem states that, under similar assumptions as in
Theorem~\ref{thm:mainthm} and the partial orthogonality condition
(Assumption \ref{assumption:po}), if features are not too strongly
correlated (as measured by the maximal eigenvalue of the covariance
matrix) and $p_n$ grows sufficiently fast, we can choose a threshold
$\gamma_n$ for hard-thresholding such that the false selection rate
becomes asymptotically negligible.

Our theorems say little about how to actually select the
hard-thresholding parameter $\gamma_n$ in practice. Following
\cite{fan08:_sure} and \cite{fan09:_ultrah_dimen_featur_selec}, we
would typically choose $\gamma_n$ such that $|\mc{M}_\mathrm{pre}^n|$
is of order $n/\log n$. Devising a general data-adaptive way of
choosing $\gamma_n$ is an open problem; false-selection-based criteria
are briefly mentioned in Section \ref{sec:scaling-fast}.

\subsection{The special case of the Aalen model}
Additional insight into the impact of censoring on FAST-SIS is
possible within the more restrictive context of the nonparametric
Aalen model with Gaussian features (\cite{aalen80};
\cite{aalen89}). This particular model asserts a hazard rate function
for $T_i$ of the form
\begin{equation}
\label{eq:def-of-aalenmodel}
  \lambda_j(t) = \lambda_0(t)+\vek{Z}_j^\top \vekg\alpha^0(t), \quad j=1,2,\ldots,n;
\end{equation}
for some baseline hazard rate function $\lambda_0$ and $\vekg{\alpha}^0$ a vector of continuous regression coefficient functions. The Aalen model extends the
Lin-Ying model of Section \ref{sec:classical} by allowing time-varying
regression coefficients. Alternatively, it can be viewed as defining
an expansion to the first order of a general hazard rate function in
the class~\eqref{eq:model-for-haz} in the sense that
\begin{equation}
\label{eq:aalen-expansion}
  \lambda\big(t,\vek{Z}_1^\top \vekg\alpha^0\big) \approx \lambda(t,0)+\vek{Z}_1^\top \vekg\alpha^0\frac{\partial \lambda(t,x)}{\partial x}\Big|_{x=0}.
\end{equation}
For Aalen models with Gaussian features, we have the following analogue to Theorem~\ref{thm:mainthm}.
\begin{theorem}
\label{thm:mainthm-aalen}
Suppose that Assumptions
\ref{assumption:monotonicity}-\ref{assumption:lr} hold alongside the
regularity conditions of the appendix. Suppose moreover that the
$\vek{Z}_1$ is mean zero Gaussian and that $T_1$ follows a model of
the form~\eqref{eq:def-of-aalenmodel} with regression coefficients
$\vekg{\alpha}^0$. Assume that $C_1$ also follows a model of the form
\eqref{eq:def-of-aalenmodel} conditionally on $\vek{Z}_1$ and that
censoring is independent. Let
$\vek{A}^0(t):=\int_0^t \vekg{\alpha}^0(s) \d s$. If for some
$\kappa<1/2$ and $c_1>0$, we have
    \begin{equation}
\label{eq:aalen-eq}
      |\cov[Z_{1j},\vek{Z}_1^\top \E\{\vek{A}^0(T_1 \land C_1 \land \tau)\}]| \geq c_1n^{-\kappa}, \quad j \in \mc{M}^n,
    \end{equation}
    then the conclusions of Theorem~\ref{thm:mainthm} hold with $\eta=2$.
\end{theorem}
In view of \eqref{eq:aalen-expansion}, Theorem~\ref{thm:mainthm-aalen}
can be viewed as establishing, within the model class
\eqref{eq:model-for-haz}, conditions for first-order validity of
FAST-SIS under a general (independent) censoring mechanism and Gaussian
features. The expectation term in \eqref{eq:aalen-eq} is essentially
the `expected regression coefficients at the exit time' which is
strongly dependent on censoring through the symmetric dependence on
survival and censoring time.

In fact, general independent censoring is a nuisance even in the
Lin-Ying model which would otherwise seem the `natural model' in which
to use FAST-SIS. Specifically, assuming only independent censoring,
suppose that $T_1$ follows a Lin-Ying model with regression
coefficients $\vekg\alpha^0$ conditionally on $\vek{Z}_1$ and that
$C_1$ also follows some Lin-Ying model conditionally on $\vek{Z}_1$. If
$\vek{Z}_1 = \vekg{\Sigma}^{1/2}\tilde{\vek{Z}}_1$ where the
components of $\tilde{\vek{Z}}_1$ are i.i.d. with mean zero and unit
variance, there exists a $p_n \times p_n$ diagonal matrix $\vek{C}$
such that
\begin{equation}
\label{eq:single-index-w-censoring}
  \vekg{\delta} = \vekg\Sigma^{1/2} \vek{C} \vekg{\Sigma}^{1/2} \vekg\alpha^0.
\end{equation}
See Lemma \ref{prop:aalen-screen} in the appendix. It holds that
$\vek{C}$ has constant diagonal iff features are Gaussian; otherwise
the diagonal is nonconstant and depends nontrivially on the regression
coefficients of the censoring model. A curious implication is
that, under Gaussian features, FAST screening has the
SIS property for this `double' Lin-Ying model irrespective of the
(independent) censoring mechanism. Conversely, sufficient conditions
for a SIS property to hold here under more general feature
distributions would require the $j$th component of
$\vekg\Sigma^{1/2} \vek{C} \vekg{\Sigma}^{1/2} \vekg\alpha^0$ to be
`large' whenever $\alpha_j^0$ is `large'; hardly a very operational
assumption. In other words, even in the simple Lin-Ying
model, unrestricted censoring complicates analysis of
FAST-SIS considerably.

\subsection{Scaling the FAST statistic}
\label{sec:scaling-fast}
The FAST statistic is easily generalized to incorporate
scaling. Inspection of the results in the appendix immediately shows
that multiplying the FAST statistic by some strictly positive,
deterministic weight does not alter its asymptotic behavior. Under
suitable assumptions, this also holds when  weights are
stochastic. In the notation of Section \ref{sec:classical}, the
following two types of scaling are immediately relevant:
\begin{alignat}{2}
\label{eq:tfast}    d_{j}^{Z} &= d_{j}B_{jj}^{-1/2} &&\textrm{ ($Z$-FAST);}\\ 
\label{eq:lyfast}    d_{j}^{\mathrm{LY}} &= d_{j}D_{jj}^{-1} && \textrm{ (Lin-Ying-FAST).}
\end{alignat}
The $Z$-FAST statistic corresponds to standardizing $\vek{d}$ by its
estimated standard deviation; screening with this statistic is
equivalent to the standard approach of ranking features according to
univariate Wald $p$-values. Various forms of asymptotic false-positive
control can be implemented for $Z$-FAST, courtesy of the central limit
theorem. Note that $Z$-FAST is model-independent in the sense that its
interpretation (and asymptotic normality) does not depend on a
specific model. In contrast, the Lin-Ying-FAST statistic is
model-specific and corresponds to calculating the univariate
regression coefficients in the Lin-Ying model, thus leading to an
analogue of the idea of `ranking by absolute regression coefficients' of
\cite{fan10:_sure_np} .

We may even devise a scaling of $\vek{d}$ which mimicks the `ranking
by marginal likelihood ratio' screening of \cite{fan10:_sure_np} by
considering univariate versions of the natural loss function
$\vekg{\beta} \mapsto \vekg{\beta}^\top \vek{D} \vekg{\beta} -2
\vekg{\beta}^\top \vek{d}$ for the Lin-Ying model. The components of
the resulting statistic are rather similar to \eqref{eq:lyfast}, taking the form
\begin{equation}
  \label{eq:lossfast}    d_{j}^{\mathrm{loss}} = d_{j}B_{jj}^{-1/2} \textrm{ (loss-FAST).}
\end{equation}

Additional flexibility can be gained by using a time-dependent scaling
where some strictly positive (stochastic) weight is multiplied
on the integrand in (\ref{eq:dev-at-event-defin}).  This is beyond the
scope of the present paper.

\section{Beyond simple independent screening -- iterated FAST screening}
\label{sec:iteratedsis}
The main assumption underlying any SIS method, including FAST-SIS, is
that the design is close to orthogonal. This assumption is easily
violated: a relevant feature may have a low marginal association with
survival; an irrelevant feature may be indirectly associated with
survival through associations with relevant features etc.  To address
such issues, \cite{fan08:_sure} and
\cite{fan09:_ultrah_dimen_featur_selec} proposed various heuristic
iterative SIS (ISIS) methods which generally work as follows. First,
SIS is used to recruit a small subset of features within which an even
smaller subset of features is selected using a (multivariate) variable
selection method such as penalized regression. Second, the (univariate) relevance
of each feature not selected in the variable selection step is
re-evaluated, adjusted for all the selected features. Third, a small
subset of the most relevant of these new features is joined to the set
of already selected features, and the variable selection step is
repeated. The last two steps are iterated until the set of selected
features stabilizes or some stopping criterion of choice is reached.

We advocate a similar strategy to extend the application domain of
FAST-SIS. In view of Section~\ref{sec:mod-based-inter}, a variable
step using a working Lin-Ying model is intuitively sensible. We may
also provide some formal justification. Firstly, estimation in a
Lin-Ying model corresponds to optimizing the loss function
\begin{equation}
\label{eq:lossf-for-ahaz}
 L(\vekg\beta):=\vekg\beta^\top \vek{D} \vekg\beta -2 \vekg\beta^\top \vek{d};
\end{equation}
where $\vek{D}$ was defined in Section \ref{sec:mod-based-inter}. As
discussed by \cite{martinussen09:_covar_selec_semip}, the loss
function \eqref{eq:lossf-for-ahaz} is meaningful for general hazard
rate models: it is the empirical version of the mean squared
prediction error for predicting, with a working Lin-Ying model, the
part of the intensity which is orthogonal to the at-risk indicator. In
the present context, we are mainly interested in the model selection
properties of a working Lin-Ying model. Suppose that $T_1$
conditionally on $\vek{Z}_1$ follows a single-index model of the
form~\eqref{eq:model-for-haz} and that Assumptions~\ref{assumption:ran-cens}-\ref{assumption:po}
hold.  Suppose that $\vekg\Delta \vekg\beta^0=\vekg\delta$ with
$\vekg\Delta$ the in probability limit of $\vekg{D}$. Then $\alpha_j^0
\equiv 0$ implies $\beta_j^0=0$ \citep{hattori06:_some} so that a
working Lin-Ying model will yield conservative model selection in a quite general setting. Under
stronger assumptions, the following result, related to work by
\cite{brillinger83:_gauss} and \cite{li89:_regres}, is available.
\begin{theorem}
\label{thm:consistent-joint-single-index}
Assume that $T_1$ conditionally on $\vek{Z}_1$ follows a single-index model of the
form \eqref{eq:model-for-haz}. Suppose moreover that
Assumption~\ref{assumption:lr} holds and that $C_1$ is independent of
$T_1,\vek{Z}_1$ (random censoring). If $\vekg\beta^0$ defined by
$\vekg\Delta \vekg\beta^0 =\vekg\delta$ is the vector of regression
coefficients of the associated working Lin-Ying model and
$\vekg\Delta$ is nonsingular, then there exists a nonzero constant
$\nu$ depending only on the distributions of $\vek{Z}_1^\top
\vekg\alpha^0$ and $C_1$ such that $\vekg\beta^0 = \nu
\vekg\alpha^0$.
\end{theorem}
Thus a working Lin-Ying model can consistently estimate regression
coefficient signs under misspecification. From the efforts of
\cite{zhu09:_variab} and \cite{zhu09:_noncon} for other types of
single-index models, it seems conceivable that variable selection
methods designed for the Lin-Ying model will enjoy certain consistency
properties within the model class \eqref{eq:model-for-haz}.  The
conclusion of Theorem~\ref{thm:consistent-joint-single-index}
continues to hold when $\vekg{\Delta}$ is replaced by any matrix
proportional to the feature covariance matrix $\vekg{\Sigma}$. This is
a consequence of Assumption~\ref{assumption:lr} and underlines the
considerable flexibility available when estimating in single-index
models.


Variable selection based on the Lin-Ying loss
\eqref{eq:lossf-for-ahaz} can be accomplished by optimizing a
penalized loss function of the form 
\begin{equation}
\label{eq:pen-loss}
  \vekg\beta \mapsto L(\vekg\beta)+\sum_{j=1}^{p}p_\lambda(|\beta_j|);
\end{equation}
where $p_\lambda\colon \mathbb{R} \to \mathbb{R}$ is some nonnegative
penalty function, singular at the origin to facilitate model selection
\citep{fan01:_variab} and depending on some tuning parameter $\lambda$
controlling the sparsity of the penalized estimator.  A popular
choice is the lasso penalty
\citep{tibshirani09:_univar_shrin_cox_model_high_dimen_data} and its
adaptive variant \citep{zou06}, corresponding to penalty functions
$p_\lambda(|\beta_j|)=\lambda |\beta_j|$ and
$p_\lambda(|\beta_j|)=\lambda |\beta_j|/|\hat{\beta}_j|$ with
$\hat{\vekg{\beta}}$ some root $n$ consistent estimator of
$\vekg{\beta}^0$, respectively. These penalties were studied by
\cite{ma07:_path} and \cite{martinussen09:_covar_selec_semip} for the
Lin-Ying model. Empirically, we have had better success with the
one-step SCAD (OS-SCAD) penalty of \cite{zou08:_one} than with 
 lasso penalties.  Letting
\begin{equation}
\label{eq:scad-pen}
w_\lambda(x) := \lambda\1(x\leq \lambda)+\frac{(a\lambda -x)_+}{a-1}\1(x>\lambda), \quad a>2
\end{equation}
an OS-SCAD penalty function for the Lin-Ying model can be defined as follows:
\begin{equation}
\label{eq:oss-scad-penalty}
  p_\lambda(|\beta_j|):= w_\lambda(\bar{D}|\hat{\beta}_j|) |\beta_j|.
\end{equation}
Here $\hat{\vekg\beta}:=\argmin_{\vekg\beta} L(\vekg\beta)$ is the
unpenalized estimator and $\bar{D}:=n^{-1}\mathrm{tr}\,(\vek{D})$ is
the average diagonal element of~$\vek{D}$; this particular re-scaling
is just one way to lessen dependency of the penalization on the time scale.  If
$\vek{D}$ has approximately constant diagonal (which is often the case
for standardized features), then re-scaling by $\bar{D}$ leads to a
similar penalty as for OS-SCAD in the linear regression model with
standardized features. The choice $a=3.7$ in \eqref{eq:scad-pen} was
recommended by \cite{fan01:_variab}.  OS-SCAD has not previously been
explored for the Lin-Ying model but its favorable
performance in ISIS for other regression models is well known
\citep{fan09:_ultrah_dimen_featur_selec,fan10:_borrow_stren}. OS-SCAD
can be implemented efficiently using, for example, coordinate descent
methods for fitting the lasso
\citep{gorst-rasmussen11:_effic,friedman07:_pathw}. For fixed~$p$, the
OS-SCAD penalty \eqref{eq:oss-scad-penalty} has the oracle property if
the Lin-Ying model holds true. A proof is beyond scope but follows by
adapting \cite{zou08:_one} along the lines of
\cite{martinussen09:_covar_selec_semip}.

In the basic FAST-ISIS algorithm proposed below, the initial recruitment step
corresponds to ranking the regression coefficients in the univariate
Lin-Ying models. This is a convenient generic choice because it enables
interpretation of the algorithm as standard `vanilla ISIS'
\citep{fan09:_ultrah_dimen_featur_selec} for the Lin-Ying model.
\begin{algorithm}[Lin-Ying-FAST-ISIS]
\label{alg:isis}
Set $\mc{M}:=\{1,\ldots,p\}$, let $r_\mathrm{max}$ be some pre-defined maximal number of iterations of the algorithm.
  \begin{enumerate}
  \item (\textit{Initial recruitment}). Perform SIS by ranking
    $|d_{j}D_{jj}^{-1}|$, $j=1,\ldots,p_n$,
    according to decreasing order of magnitude and retaining the $k_0 \leq d$ most
    relevant features~$\mc{A}_1 \subseteq \mc{M}$.
\item For $r=1,2,\ldots$ do:
  \begin{enumerate}
  \item (\textit{Feature selection}). Define $\omega_j:=\infty$ if $j \in \mc{A}_r$ and $\omega_j:=0$ otherwise. Estimate 
    \begin{displaymath}
      \hat{\vekg{\beta}}:=\argmin_{ \vekg\beta}\Big\{L(\vekg{\beta})+\sum_{j=1}^{p_n} \omega_jp_{\hat\lambda}(|\beta_j|) \Big\};
    \end{displaymath}
    with $p_\lambda$ defined in \eqref{eq:oss-scad-penalty} for some suitable tuning parameter $\hat\lambda$. Set
    $\mc{B}_r := \{j\,:\,\hat{\beta}_j\neq 0\}$.
\item If $r>1$ and $\mc{B}_r = \mc{B}_{r-1}$, or if
  $r=r_{\mathrm{max}}$; return $\mc{B}_r$. 
\item (\textit{Re-recruitment}). Otherwise, re-evaluate  relevance of
  features in $\mc{M}\backslash \mc{B}_r$ according to the
  absolute value of their regression coefficient $|\tilde{\beta}_j|$ in the
  $|\mc{M}\backslash \mc{B}_r|$ unpenalized Lin-Ying models including each feature in $\mc{M}\backslash
  \mc{B}_r$ and all features in $\mc{B}_r$, i.e.
  \begin{equation}
\label{eq:fulladj}
    \tilde{\beta}_j :=\hat{\beta}_1^{(j)}, \quad \textrm{where } \hat{\vekg\beta}^{(j)}=\argmin_{\vekg\beta_{\{j\} \cup \mc{B}_r}} L(\vekg\beta_{\{j\} \cup \mc{B}_r}), \quad j \in \mc{M}\backslash \mc{B}_r.
  \end{equation}
  Take $\mc{A}_{r+1}:=\mc{C}_r \cup \mc{B}_r$ where
  $\mc{C}_r$ is the set of the $k_r$ most relevant features in
  $\mc{M}\backslash \mc{A}_r$, ranked according to decreasing order of
  magnitude of $|\tilde{\beta}_j|$.
  \end{enumerate}
\end{enumerate}
\end{algorithm}
\cite{fan08:_sure} recommended choosing $d$ to be of order $n/\log
n$. Following \cite{fan09:_ultrah_dimen_featur_selec}, we may take
$k_0=\lfloor 2d/3 \rfloor$ and $k_l = d-|\mc{A}_l|$ at each step. This
$k_0$ ensures that we complete at least one iteration of the
algorithm; the choice of $k_r$ for $r>0$ ensures that at most $d$
features are included in the final solution.

Algorithm \ref{alg:isis} defines an iterated variant of SIS with the
Lin-Ying-FAST statistic (\ref{eq:lyfast}). We can devise an analogous
iterated variant of $Z$-FAST-SIS in which the initial recruitment is
performed by ranking based on the statistic (\ref{eq:tfast}), and the
subsequent re-recruitments are performed by ranking $|Z|$-statistics
in the multivariate Lin-Ying model according to decreasing order of
magnitude, using the variance
estimator~\eqref{eq:lin-ying-rootn-consistency}. A third option would
be to base recruitment on (\ref{eq:lossfast}) and re-recruitments on the decrease in the multivariate loss 
\eqref{eq:lossf-for-ahaz} when joining a given feature to the set of
features picked out in the variable selection step.

The re-recruitment step (b.iii) in Algorithm \ref{alg:isis} resembles
that of \cite{fan09:_ultrah_dimen_featur_selec}. Its naive
implementation will be computationally burdensome when $p_n$ is large,
requiring a low-dimensional matrix inversion per feature. Significant
speedup over the naive implementation is possible via the matrix
identity
\begin{equation}
\label{eq:matrixid}
 \vek{D}=\left(\begin{matrix}
    e & \vek{f}^\top \\
    \vek{f} & \tilde{\vek{D}}
  \end{matrix}\right) \Rightarrow \vek{D}^{-1} = \left(\begin{matrix}
    k^{-1} &  -k^{-1}\vek{f}^\top \tilde{\vek{D}}^{-1} \\
    -k^{-1} \tilde{\vek{D}}^{-1}\vek{f}& (\tilde{\vek{D}}-e^{-1}\vek{f}\vek{f}^\top)^{-1}
  \end{matrix}\right) \quad  \textrm{where } k=e-\vek{f}^\top \tilde{\vek{D}}^{-1} \vek{f}.
\end{equation}
Note that only the first row
of $\vek{D}^{-1}$ is required for the re-recruitment step so that
\eqref{eq:fulladj} can be implemented using just a single low-dimensional
matrix inversion alongside $O(p_n)$ matrix/vector multiplications. Combining
\eqref{eq:matrixid} with \eqref{eq:lin-ying-rootn-consistency}, a
similarly efficient implementation applies for $Z$-FAST-ISIS.

The variable selection step (b.i) of Algorithm \ref{alg:isis} requires
the choice of an appropriate tuning parameter. This is traditionally a
difficult part of penalized regression, particularly  when the aim is
model selection where methods such as cross-validation are prone to
overfitting \citep{leng07}. Previous work on ISIS used the Bayesian
information criterion (BIC) for tuning parameter selection
\citep{fan09:_ultrah_dimen_featur_selec}. Although BIC is based on the
likelihood, we may still define the following `pseudo BIC' based on
the loss  \eqref{eq:lossf-for-ahaz}:
\begin{equation}
\label{eq:def-of-pbic}
  \mathrm{PBIC}(\lambda) = \kappa\{L(\hat{\vekg\beta}_\lambda)-L(\hat{\vekg\beta})\}+ n^{-1} \mathrm{df}_\lambda \log n.
\end{equation}
Here $\hat{\vekg\beta}_\lambda$ is the penalized estimator,
$\hat{\vekg\beta}$ is the unpenalized estimator, $\kappa>0$ is a
scaling constant of choice, and $\mathrm{df}_\lambda$ estimates the degrees of freedom of the penalized estimator. A computationally
convenient choice is $\mathrm{df}_\lambda =
\|\hat{\vekg\beta}_\lambda\|_0$ \citep{zou07l}. It turns out that
choosing $\hat{\lambda}= \argmin_\lambda \mathrm{PBIC}_\lambda$ may lead to
model selection consistency. Specifically, the loss 
\eqref{eq:lossf-for-ahaz} for the Lin-Ying model is of the
least-squares type. Then we can repeat the arguments  of
 \cite{wang07:_unified_lasso_estim_least_squar_approx} and
show that, under suitable consistency assumptions  for the penalized estimator, there exists a sequence $\lambda_n \to 0$ yielding selection
consistency for $\hat{\vekg\beta}_{\lambda_n}$ and
satisfying
\begin{equation}
\label{eq:cons-property}
  \ssh\Big\{\inf_{\lambda \in S}
\mathrm{PBIC}(\lambda)>\mathrm{PBIC}(\lambda_n) \Big\} \to 1, \qquad n \to \infty;
\end{equation}
with $S$ the union of the set of tuning parameters $\lambda$ which
lead to overfitted (strict supermodels of the true model),
respectively underfitted models (any model which do not include the
true model). While \eqref{eq:cons-property} holds independently of
the scaling constant $\kappa$, the finite-sample behavior of PBIC
depends strongly on~$\kappa$. A sensible value may be inferred heuristically as
follows: the range of a `true' likelihood BIC is asymptotically
equivalent to a Wald statistic in the sense that (for fixed 
$p$),
\begin{equation}
\label{eq:bicexpansion}
  \mathrm{BIC}(0)-\mathrm{BIC}(\infty) =\hat{\vekg\beta}_{\mathrm{ML}}^\top
\mc{I}(\vekg\beta_0)\hat{\vekg\beta}_{\mathrm{ML}}+o_p(n^{-1/2});
\end{equation}
with $\hat{\vekg\beta}_{\mathrm{ML}}$ the maximum likelihood estimator
and $\mc{I}(\vekg\beta_0) \approx
n^{-1}\var(\hat{\vekg\beta}_\mathrm{ML}-\vekg\beta_0)^{-1}$ the
information matrix.  We may specify $\kappa$ by requiring that
$\mathrm{PBIC}(0)-\mathrm{PBIC}(\infty)$ admits an analogous
interpretation as a Wald statistic. Since $
\mathrm{PBIC}(0)-\mathrm{PBIC}(\infty) = \kappa\vek{d}^\top
\vek{D}^{-1} \vek{d} + o_p(n^{-1/2})$, it follows from
(\ref{eq:lin-ying-rootn-consistency}) that we should choose
\begin{displaymath}
\kappa := \frac{\vek{d}^\top \vek{B}^{-1}\vek{d}}{\vek{d}^\top \vek{D}^{-1} \vek{d}}.
\end{displaymath}
This choice of $\kappa$  also removes the dependency of PBIC on the time scale. 

\section{Simulation studies}
\label{sec:numerical-results}
In this section, we investigate the performance of FAST screening on
simulated data. Rather than comparing with popular variable selection methods
such as the lasso, we will compare with analogous
screening methods based on the Cox model
\citep{fan10:_borrow_stren}. This seems a more pertinent benchmark
since previous work has already demonstrated that (iterated)
SIS can outperform variable
selection based on penalized regression in a number of cases
(\cite{fan08:_sure}; \cite{fan09:_ultrah_dimen_featur_selec}).

For all the simulations, survival times were generated from three
different conditionally exponential models of the generic form
(\ref{eq:model-for-haz}); that is, a time-independent hazard `link function' applied to
a linear functional of features. For suitable
constants $c$, the link functions were as follows (see also
Figure \ref{fig:linkfun}):
\begin{displaymath}
  \begin{array}{rrcl}
      \mathrm{Logit}:& \lambda_\mathrm{logit}(t,x)&:=&\{1+\exp(c_\mathrm{logit}x\}^{-1} \\
    \mathrm{Cox}:& \lambda_\mathrm{cox}(t,x)&:=&\exp(c_\mathrm{cox}x) \\
    \mathrm{Log}:& \lambda_\mathrm{log}(t,x)&:=& \log\{\mathrm{e}+(c_\mathrm{log}x)^2\}\{1+\exp(c_\mathrm{log}x)\}^{-1}.
  \end{array}
\end{displaymath}
The link functions represent different characteristic effects on the
feature functional, ranging from uniformly bounded (logit) over fast
decay/increase (Cox), to fast decay/slow increase (log). We took
$c_\mathrm{logit}=1.39$, $c_\mathrm{cox}=0.68$, and
$c_\mathrm{\log}=1.39$ and, unless otherwise stated, survival times
were right-censored by independent exponential random variables with
rate parameters 0.12 (logit link), 0.3 (Cox link) and 0.17 (log
link). These constants were selected to provide a crude `calibration'
to make the simulation models more comparable: for a univariate
standard Gaussian feature $Z_1$, a regression coefficient $\beta=1$,
and a sample size of $n=300$, the expected $|Z|$-statistic was 8 for
all three link functions with an expected censoring rate of~25\%, as
evaluated by numerical integration based on the true likelihood.

Methods for FAST screening have been implemented  in the  R package
`ahaz' \citep{gorst-rasmussen11}.

\begin{figure}[h!]
\center
   \includegraphics[scale=.75]{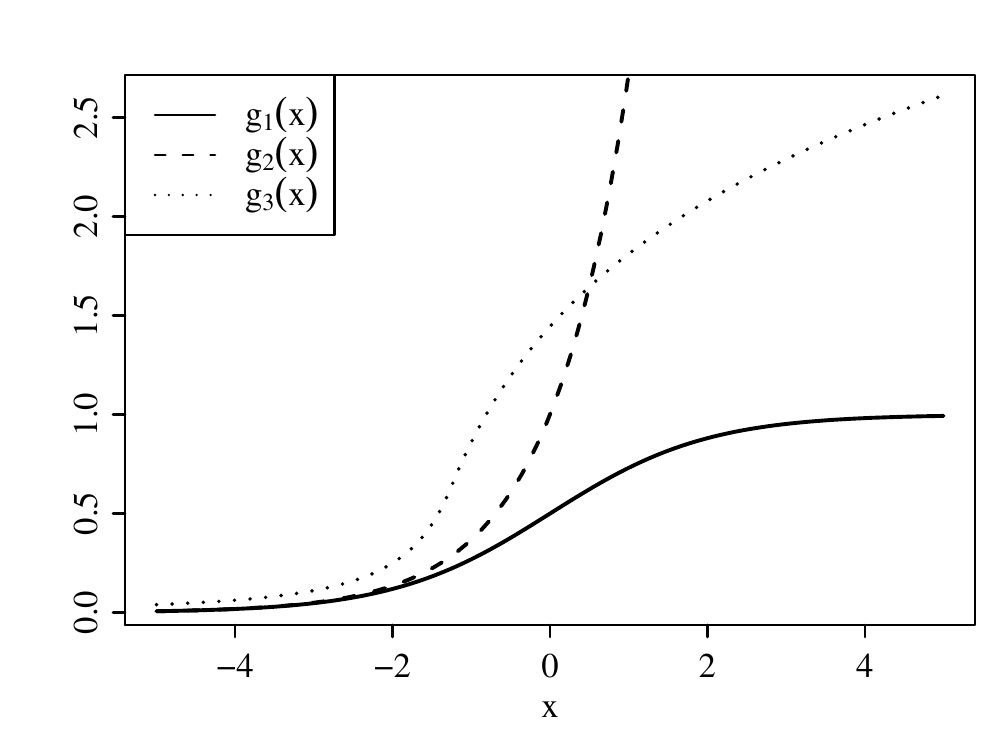}
\caption{\label{fig:linkfun} The three hazard rate link functions used in the simulation studies}
\end{figure}

\subsection{Performance of FAST-SIS}
\label{sec:is-weighting-useful}
We first considered the performance of basic, non-iterated FAST-SIS.  Features were generated as in scenario~1 of \cite{fan10:_sure_np}. Specifically, let
$\epsilon$ be standard Gaussian. Define 
\begin{equation}
\label{eq:sis-scenario}
  Z_{1j} := \frac{\epsilon_j+a_j\epsilon}{\sqrt{1+a_j^2}}, \quad j=1,\ldots,p;
\end{equation}
where $\epsilon_j$ is independently distributed as a standard
Gaussian for $j=1,2,\ldots,\lfloor p/3\rfloor$: independently
distributed according to a double exponential distribution with
location parameter zero and scale parameter 1 for $j=\lfloor p/3
\rfloor +1,\ldots,\lfloor 2p/3 \rfloor$; and independently distributed
according to a Gaussian mixture $0.5N(-1,1)+0.5N(1,0.5)$ for
$j=\lfloor 2p/3\rfloor+1,\ldots,p$. The constants $a_j$ satisfy
$a_1=\cdots=a_{15}$ and $a_j=0$ for $j>15$. With the choice
$a_1=\sqrt{\rho/(1-\rho)}$, $0 \leq \rho \leq 1$, we obtain
$\cor(Z_{1i},Z_{1j})=\rho$ for $i \neq j$, $i,j \leq 15$, enabling
crude adjustment of the correlation structure of the feature
distribution. Regression coefficients were chosen to be of the generic
form $ \vekg\alpha^0=(1,1.3,1,1.3,\ldots)^\top$ with exactly the first
$s$ components nonzero.

\begin{table}
\caption{MMMS and RSD (in parentheses) for basic SIS  with $n=300$ and $p=20,000$ (100  simulations).\label{tab:tab1}}
\centering
\fbox{%
{\footnotesize
\begin{tabular}{llrrrrrrrrrrr}
\\[-7pt]
&&\multicolumn{3}{c}{$\lambda_\mathrm{logit}$  }&& \multicolumn{3}{c}{$\lambda_\mathrm{cox}$ } && \multicolumn{3}{c}{$\lambda_\mathrm{log}$ } \\[5pt]
  \cline{3-5} \cline{7-9} \cline{11-13}\\[-5pt]
$\rho$ &  & $s=3$ & $s=6$ & $s=9$ &&  $s=3$ & $s=6$ & $s=9$ &&  $s=3$ & $s=6$ & $s=9$   \\[5pt]
\hline
\\[-5pt]
0   & $\vek{d}$& 3 (1) & 32 (53) & 530 (914) & & 3 (0) & 7 (5) & 45 (103) & & 3 (0) & 22 (44) & 202 (302)  \\ 
   & $\vek{d}^{\mathrm{LY}}$& 4 (1) & 66 (95) & 678 (939) & & 3 (0) & 11 (14) & 96 (176) & & 3 (1) & 41 (87) & 389 (466)  \\ 
   & $\vek{d}^{Z}$& 3 (1) & 40 (71) & 522 (873) & & 3 (0) & 7 (7) & 48 (105) & & 3 (0) & 22 (45) & 262 (318)  \\ 
   & \textbf{Cox}& 3 (1) & 44 (68) & 572 (928) & & 3 (0) & 7 (4) & 40 (117) & & 3 (0) & 26 (51) & 280 (306)  \\[4pt] 
0.25   & $\vek{d}$& 3 (0) & 6 (1) & 11 (1) & & 3 (0) & 6 (0) & 9 (1) & & 3 (0) & 6 (1) & 10 (1)  \\ 
   & $\vek{d}^{\mathrm{LY}}$& 3 (0) & 7 (1) & 11 (2) & & 3 (0) & 6 (1) & 10 (1) & & 3 (0) & 7 (1) & 11 (1)  \\ 
   & $\vek{d}^{Z}$& 3 (0) & 6 (1) & 11 (1) & & 3 (0) & 6 (0) & 10 (1) & & 3 (0) & 6 (1) & 10 (1)  \\ 
   & \textbf{Cox}& 3 (0) & 6 (1) & 11 (1) & & 3 (0) & 6 (0) & 9 (1) & & 3 (0) & 6 (1) & 10 (1)  \\[4pt] 
0.5   & $\vek{d}$& 3 (0) & 7 (2) & 12 (2) & & 3 (0) & 6 (1) & 10 (1) & & 3 (0) & 7 (1) & 11 (2)  \\ 
   & $\vek{d}^{\mathrm{LY}}$& 3 (0) & 9 (3) & 13 (1) & & 3 (0) & 8 (2) & 13 (2) & & 3 (0) & 8 (2) & 12 (2)  \\ 
   & $\vek{d}^{Z}$& 3 (0) & 8 (3) & 12 (1) & & 3 (0) & 7 (2) & 12 (2) & & 3 (0) & 7 (2) & 12 (2)  \\ 
   & \textbf{Cox}& 3 (1) & 9 (3) & 13 (2) & & 3 (0) & 6 (1) & 11 (2) & & 3 (0) & 8 (2) & 12 (2)  \\[4pt] 
0.75   & $\vek{d}$& 3 (1) & 9 (2) & 13 (1) & & 3 (0) & 8 (2) & 12 (1) & & 3 (1) & 9 (3) & 12 (2)  \\ 
   & $\vek{d}^{\mathrm{LY}}$& 4 (2) & 11 (3) & 14 (2) & & 4 (1) & 11 (3) & 14 (1) & & 4 (2) & 10 (2) & 13 (1)  \\ 
   & $\vek{d}^{Z}$& 4 (1) & 10 (2) & 13 (1) & & 3 (1) & 10 (3) & 13 (1) & & 3 (1) & 9 (2) & 13 (1)  \\ 
   & \textbf{Cox}& 5 (3) & 12 (2) & 14 (1) & & 3 (0) & 7 (2) & 12 (2) & & 4 (1) & 11 (3) & 14 (2)  \\[5pt]
  \end{tabular} 
}
}
\end{table}
\begin{figure}[h!]
\center
   \includegraphics[scale=.7]{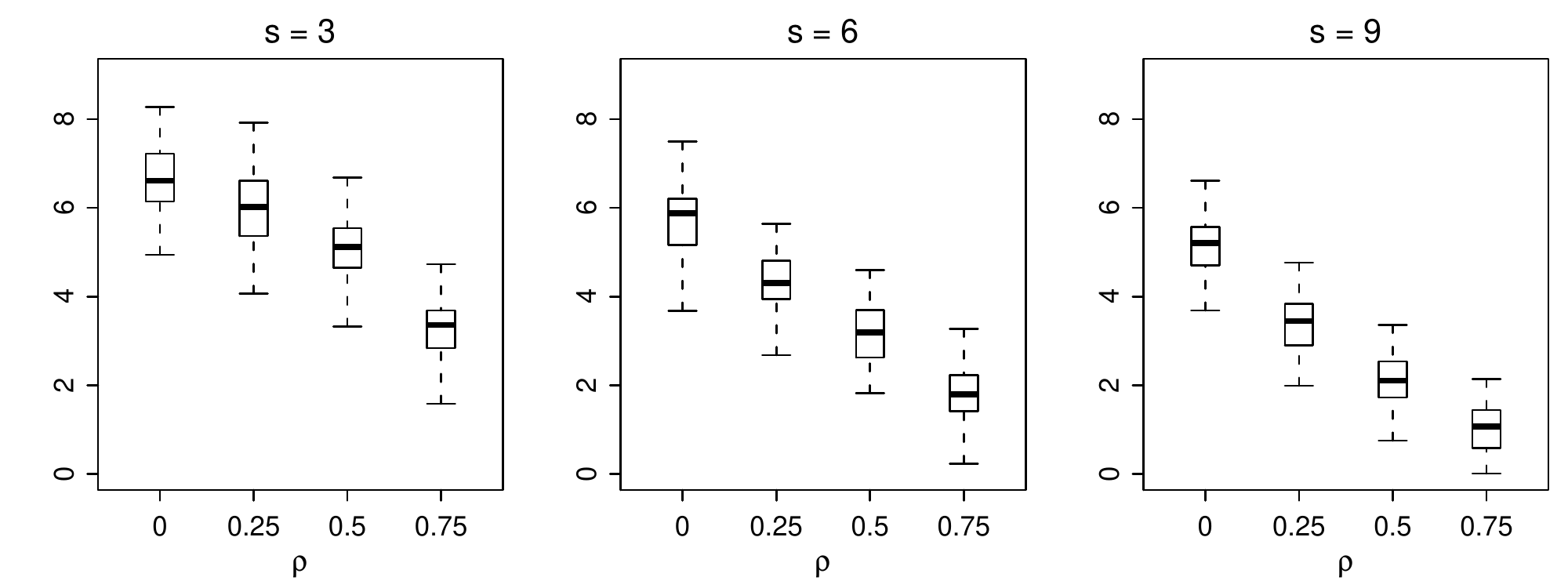}
   \caption{\label{fig:box1} Minimum observed 
     $|Z|$-statistics in  the oracle model under
     $\lambda_\mathrm{log}$, for the SIS simulation study.}
\end{figure}

For each combination of hazard link function, non-sparsity level $s$, and
correlation $\rho$, we performed 100 simulations with $p=20,000$
features and $n=300$ observations. Features were ranked using the
vanilla FAST statistic, the scaled FAST statistics (\ref{eq:tfast})
and (\ref{eq:lyfast}), and SIS based on a Cox working model
(Cox-SIS), the latter ranking features according their absolute univariate regression
coefficient.  Results are shown in Table~\ref{tab:tab1}. As a
performance measure, we report the median of the minimum model size
(MMS) needed to detect all relevant features alongside its relative
standard deviation (RSD), the interquartile range divided by 1.34. The
MMS is a useful performance measure for this type of study since it
eliminates the need to select a threshold parameter for SIS. The
censoring rate in the simulations was typically 30\%-40\%.

For all methods, the MMMS is seen to increase with feature
correlation $\rho$ and non-sparsity~$s$. As
also noted by \cite{fan10:_sure_np} for the case of SIS for
generalized linear models, some correlation among features can
actually be helpful since it increases the strength of marginal
signals. Overall, the statistic $\vek{d}^\mathrm{\mathrm{LY}}$ seems
to perform slightly worse than both $\vek{d}$ and
$\vek{d}^\mathrm{Z}$ whereas the latter two statistics perform
similarly to Cox-SIS. In our basic implementation, screening with any
of the FAST statistics was more than 100 times faster than Cox-SIS,
providing a rough indication of the relative computational efficiency
of FAST-SIS.

To gauge the relative difficulty of the different simulation
scenarios, Figure \ref{fig:box1} shows box plots of the minimum of the
observed $|Z|$-statistics in the oracle model (the joint model with
only the relevant features included and estimation based on the
likelihood under the true link function) for the link function
$\lambda_\mathrm{log}$. This particular link function represents an
`intermediate' level of difficulty; with $|Z|$-statistics for
$\lambda_\mathrm{cox}$ generally being somewhat larger and
$|Z|$-statistics for $\lambda_\mathrm{logit}$ being slightly
smaller. Even with oracle information and the correct working model,
these are evidently difficult data to deal with.

\subsection{FAST-SIS with non-Gaussian features and
  nonrandom censoring}
\label{sec:non-random-censoring}
We next investigated FAST-SIS with non-Gaussian features and a more
complex censoring mechanism.  The simulation scenario
was inspired by the previous section but with all features generated
according to either a standard Gaussian distribution, a
$t$-distribution with~4 degrees of freedom, or a unit rate exponential
distribution. Features were standardized to have
mean zero and variance one, and the feature correlation structure was such
that $\mathrm{Cor}(Z_{1i},Z_{1j})=0.125$ for $i,j < 15$, $i \neq j$
and $\mathrm{Cor}(Z_{1i},Z_{1j})=0$ otherwise. Survival times were
generated according to the link function $\lambda_\mathrm{log}$ with
regression coefficients $\vekg{\beta}=(1,1.3,1,1.3,1,1.3,0,0,\ldots)$
while censoring times were generated according to the same model (link
function $\lambda_\mathrm{log}$ and conditionally on the same feature
realizations) with regression coefficients
$\tilde{\vekg{\beta}}=k\vekg{\beta}$. The constant $k$
controls the association between censoring and survival
times, leading to a basic example of nonrandom censoring
(competing risks).

Using $p=20,000$ features and~$n=300$ observations, we performed 100
simulations under each of the three feature distributions, for
different values of $k$. Table \ref{tab:tab-cens} reports the MMMS and
RSD for the four different screening methods of the previous section,
as well as for the statistic $\vek{d}^{\mathrm{loss}}$ in
(\ref{eq:lossfast}). The censoring rate in all scenarios was around
50\%.

From the column with $k=0$ (random censoring), the heavier tails of
the $t$-distribution increases the MMMS, particularly for
$\vek{d}^{\mathrm{LY}}$.  The vanilla FAST statistic $\vek{d}$ seems
the least affected here, most likely because it does not directly
involve second-order statistics which are poorly estimated due to the
heavier tails. While $\vek{d}^{Z}$ and $\vek{d}^{\mathrm{loss}}$ are
also scaled by second-order statistics, the impact of the tails is
dampened by the square-root transformation in the scaling factors. In
contrast, the more distinctly non-Gaussian exponential distribution is
problematic for $\vek{d}^{Z}$. Overall, the statistics~$\vek{d}$ and
$\vek{d}^{\mathrm{loss}}$ seems to have the best and most consistent
performance across feature distributions. Nonrandom censoring
generally increases the MMMS and RSD, particularly for the
non-Gaussian distributions. There appears to be no clear difference
between the effect of positive and negative values of $k$. We found
that the effect of $k \neq 0$ diminished when the sample size was
increased (results not shown), suggesting that nonrandom censoring in
the present example leads to a power rather than bias issue. This may
not be surprising in view of the considerations below
\eqref{eq:single-index-w-censoring}. However, the example still shows
the dramatic impact of nonrandom censoring on the performance of~SIS.

\begin{table}
  \caption{\label{tab:tab-cens}MMMS and RSD (in parentheses) for SIS under nongaussian features/nonrandom censoring with $n=300$ and $p=20,000$ (100 simulations).}
\centering
\fbox{%
{\footnotesize
  \begin{tabular}{llcccrrrr}
\\[-7pt]
&&&&&\multicolumn{4}{c}{$k$} \\[5pt]
\cline{6-9}\\[-5pt]
\textsl{Feature distr.} & && $k=0$ && $-0.5$ & $-0.25$ & $0.25$ & $0.5$\\[5pt]
\hline
\\[-5pt]

\textsl{Gaussian}   & $\vek{d}$&& 6 (1) && 8 (8) & 7 (4) & 6 (1) & 7 (3) \\ 
& $\vek{d}^{\mathrm{LY}}$&& 6 (1) && 8 (6) & 7 (3) & 7 (2) & 8 (5) \\ 
& $\vek{d}^{Z}$&& 6 (1) && 7 (6) & 7 (2) & 6 (1) & 7 (2) \\ 
& $\vek{d}^{\mathrm{loss}}$&& 6 (1) && 8 (6) & 7 (3) & 6 (1) & 7 (3) \\ 
   & \textbf{Cox} && 6 (1) && 8 (5) & 7 (2) & 6 (1) & 7 (2) \\[4pt]  
$t$ ($df=4$)    & $\vek{d}$&& 6 (1) && 13 (17) & 7 (5) & 6 (1) & 7 (3) \\ 
   & $\vek{d}^{\mathrm{LY}}$&& 11 (7) && 12 (8) & 9 (7) & 48 (136) & 99 (185) \\ 
   & $\vek{d}^{Z}$&& 7 (3) && 17 (20) & 8 (5) & 7 (2) & 7 (3) \\ 
   & $\vek{d}^{\mathrm{loss}}$&& 6 (1) && 8 (7) & 7 (4) & 8 (15) & 10 (10) \\ 
   & \textbf{Cox} && 7 (4) && 15 (23) & 8 (10) & 8 (4) & 9 (5) \\[4pt]  
\textsl{Exponential}   & $\vek{d}$&& 6 (1) && 6 (2) & 6 (1) & 7 (4) & 8 (7) \\ 
   & $\vek{d}^{\mathrm{LY}}$&& 6 (1) && 11 (12) & 7 (3) & 6 (1) & 6 (1) \\ 
   & $\vek{d}^{Z}$&& 15 (10) && 34 (36) & 24 (17) & 22 (28) & 26 (29) \\ 
   & $\vek{d}^{\mathrm{loss}}$&& 6 (0) && 7 (4) & 6 (1) & 6 (1) & 6 (1) \\ 
   & \textbf{Cox} && 8 (4) && 22 (31) & 14 (11) & 9 (6) & 9 (8) \\[5pt]  
  \end{tabular}
 }}
\end{table}

\subsection{Performance of FAST-ISIS}
We lastly evaluated the ability of FAST-ISIS (Algorithm
\ref{alg:isis}) to cope with scenarios where FAST-SIS fails. As in
the previous sections, we compare our results with the analogous ISIS
screening method for the Cox model. To perform Cox-ISIS, we used the R
package `SIS', with (re)recruitment based on the absolute Cox
regression coefficients and variable selection based on OS-SCAD. We
also compared with $Z$-FAST-ISIS variant described below Algorithm
\ref{alg:isis} in which (re)recruitment is based on the Lin-Ying model
$|Z|$-statistics (results for FAST-ISIS with (re)recruitment based on the
loss function were very similar).

For the simulations, we adopted the structural form of the feature
distributions used by \cite{fan10:_borrow_stren}. We
considered $n=300$ observations and $p=500$ features which were
jointly Gaussian and marginally standard Gaussian. Only regression
coefficients and feature correlations differed between cases as
follows:
\begin{enumerate}
\item The regression coefficients are $\beta_1=-0.96$,
  $\beta_2=0.90$, $\beta_3=1.20$, $\beta_4=0.96$, $\beta_5=-0.85$,
  $\beta_6=1.08$ and $\beta_j=0$ for $j>6$.  Features are
  independent, $\cor(Z_{1i},Z_{1j})=0$ for $i \neq j$.
\item The  regression coefficients are the same as in (a) while $\mathrm{Corr}(Z_{1i},Z_{1j})=0.5$ for $i \neq j$.
\item The regression coefficients are
  $\beta_1=\beta_2=\beta_3=4/3$, $\beta_4=-2\sqrt{2}$. The correlation
  between features is given by $\cor(Z_{1,4},Z_{1j})=1/\sqrt{2}$ for $j \neq
  4$ and $ \cor(Z_{1i},Z_{1j})=0.5$ for $i \neq j$, $i,j\neq 4$.
\item The regression coefficients are
  $\beta_1=\beta_2=\beta_3=4/3$, $\beta_4=-2\sqrt{2}$ and
  $\beta_5=2/3$. The correlation between features is 
  $\cor(Z_{1,4},Z_{1j})=1/\sqrt{2}$ for $j \notin \{4,5\}$, $\cor(Z_{1,5},Z_{1j})=0$
  for $j \neq 5$, and $ \cor(Z_{1i},Z_{1j})=0.5$ for $i \neq j$, $i,j \notin \{4,5\}$.
\end{enumerate}
Case (a) serves as a basic benchmark whereas case (b) is harder
because of the correlation between relevant and irrelevant
features. Case (c) introduces a strongly relevant feature $Z_4$ which
is not marginally associated with survival; lastly, case (d) is
similar to case (c) but also includes a feature $Z_5$ which is weakly
associated with survival and does not `borrow' strength from its
correlation with other relevant features.

Following \cite{fan10:_borrow_stren}, we took $d=\lfloor n/\log n/3\rfloor
= 17$ for the initial dimension reduction; performance did not depend
much on the detailed choice of $d$ of order $n/\log n$. For the three different screening
methods, ISIS was run for maximum of 5 iterations. (P)BIC was used for
tuning the variable selection steps. Results are shown in Table
\ref{tab:table3}, summarized over 100 simulations. We report the
average number of truly relevant features selected by ISIS and the average
final model size, alongside standard deviations in parentheses. To
provide an idea of the improvement over basic SIS, we also report the
median of the minimum model size (MMMS) for the initial SIS step
(based on vanilla FAST-SIS only). The censoring rate in the different
scenarios was 25\%-35\%.

The overall performance of the three ISIS methods is comparable
between the different cases. All methods deliver a dramatic
improvement over non-iterated SIS, but no one method performs
significantly better than the others. The two FAST-ISIS methods have a
surprisingly similar performance. As one would expect, Cox-ISIS does
particularly well under the link function $\lambda_\mathrm{cox}$ but
does not appear to be uniformly better than the two FAST-ISIS methods even in
this ideal setting.  Under the link function $\lambda_\mathrm{logit}$,
both FAST-ISIS methods outperform Cox-ISIS in terms of the number of
true positives identified, as do they for the link function
$\lambda_\mathrm{log}$, although less convincingly. On the other hand,
the two FAST-ISIS methods generally select slightly larger models
than Cox-ISIS and their false-positive rates (not shown) are
correspondingly slightly larger. FAST-ISIS was 40-50
times faster than Cox-ISIS, typically completing calculations in 0.5-1
seconds in our specific implementation. Figure \ref{fig:box2} shows
box plots of the minimum of the observed $|Z|$-statistics in the
oracle model (based on the likelihood undebr the true model).

\begin{table}
\caption{\label{tab:table3}Simulation results for ISIS with $n=300$, $p=500$ and $d=17$ (100 simulations). Numbers in parentheses are standard deviations (or relative standard deviation, for the MMMS). }
\centering
\fbox{%
{\footnotesize
\begin{tabular}{rrrrrrrrrrrr}
\\[-7pt]
 &&&& \multicolumn{3}{c}{\textsl{Average no.~true positives (ISIS)}} && \multicolumn{3}{c}{\textsl{Average model size (ISIS)}} \\[5pt]
 \cline{5-7} \cline{9-11}\\[-5pt]
 \textsl{Link} & \textsl{Case} & \textsl{MMMS (RSD)} &&   LY-FAST & $Z$-FAST & Cox &&   LY-FAST & $Z$-FAST & Cox   \\[5pt]
 \hline
\\[-5pt]
$\lambda_\mathrm{logit}$   & (a)& 7 (3) & & 6.0 (0) & 6.0 (0) & 5.5 (1) & & 7.8 (1) & 7.9 (2) & 6.3 (2) \\ 
   & (b)& 500 (1) & & 5.5 (1) & 5.5 (1) & 3.4 (1) & & 7.0 (2) & 6.7 (2) & 4.3 (2) \\ 
   & (c)& 240 (125) & & 3.7 (1) & 3.8 (1) & 3.0 (2) & & 5.2 (2) & 5.7 (3) & 4.5 (4) \\ 
   & (d)& 230 (124) & & 4.8 (1) & 4.7 (1) & 3.5 (2) & & 5.9 (2) & 6.2 (3) & 4.9 (4) \\[4pt] 
$\lambda_\mathrm{cox}$   & (a)& 7 (1) & & 6.0 (0) & 6.0 (0) & 6.0 (0) & & 7.5 (1) & 7.5 (1) & 6.2 (1) \\ 
   & (b)& 500 (1) & & 5.8 (1) & 5.8 (1) & 5.6 (1) & & 7.2 (2) & 6.8 (1) & 6.4 (2) \\ 
   & (c)& 218 (120) & & 3.7 (1) & 3.6 (1) & 3.0 (2) & & 5.1 (3) & 5.3 (3) & 4.9 (4) \\ 
   & (d)& 258 (129) & & 4.9 (1) & 4.8 (1) & 3.8 (2) & & 6.3 (2) & 6.0 (2) & 6.4 (5) \\[4pt]  
$\lambda_\mathrm{log}$   & (a)& 6 (1) & & 6.0 (0) & 6.0 (0) & 6.0 (0) & & 7.3 (1) & 7.4 (1) & 6.3 (1) \\ 
   & (b)& 500 (1) & & 5.8 (1) & 5.7 (1) & 4.9 (1) & & 7.2 (2) & 6.7 (1) & 5.7 (2) \\ 
   & (c)& 252 (150) & & 3.9 (0) & 3.9 (1) & 3.4 (1) & & 5.3 (2) & 4.9 (2) & 5.5 (5) \\ 
   & (d)& 223 (132) & & 4.9 (1) & 4.8 (1) & 4.0 (2) & & 6.0 (2) & 6.1 (2) & 5.9 (5) \\[5pt] 
  \end{tabular} 
}}
\end{table}

\begin{figure}[h!]
\center
   \includegraphics[scale=.7]{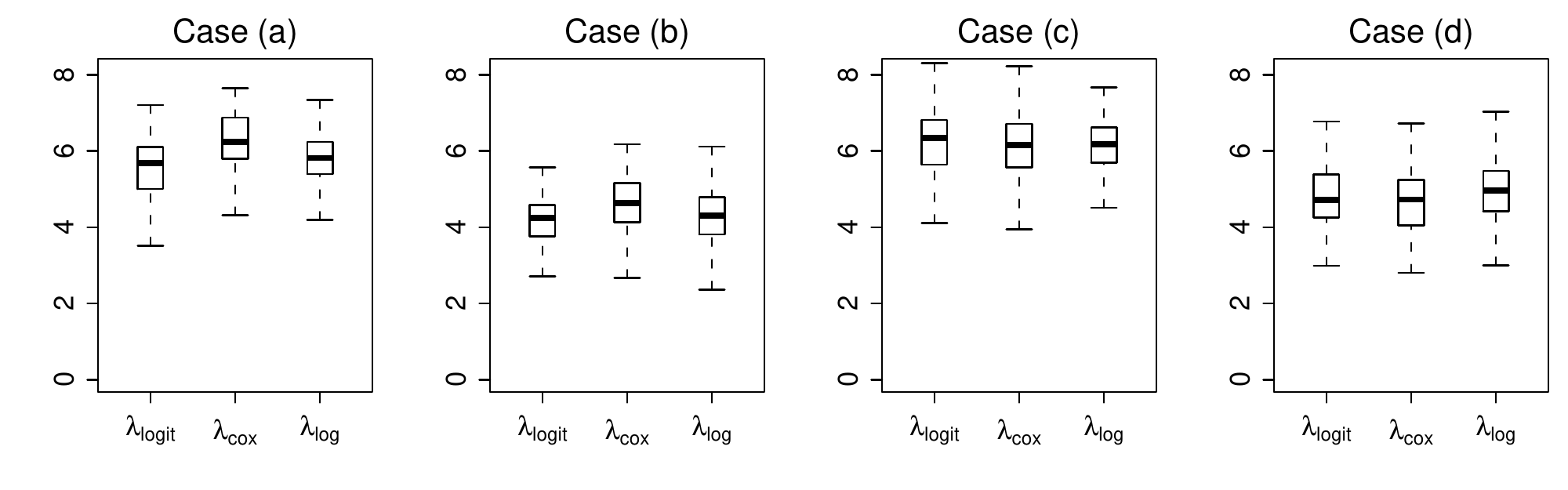}
   \caption{\label{fig:box2} Minimum observed $|Z|$-statistics
     in the oracle models for the FAST-ISIS simulation study.}
\end{figure}
\pagebreak
We have experimented with other link functions and feature
distributions than those described above (results not shown). Generally, we
found that Cox-ISIS performs worse than FAST-ISIS for bounded link
functions. The observation from Table \ref{tab:table3}, that FAST-ISIS
may improve  upon Cox-ISIS even under the link function
$\lambda_\mathrm{cox}$, does not necessarily hold when the signal
strength is increased. Then Cox-ISIS will be
superior, as expected. Changing the feature distribution to
one for which the linear regression property (Assumption
\ref{assumption:lr}) does not hold leads to a decrease in the overall
performance for all three ISIS methods.


\section{Application to AML data}

The study by \cite{metzeler08} concerns the development and evaluation
of a prognostic gene expression marker for overall survival among
patients diagnosed with cytogenetically normal acute myeloid leukemia
(CN-AML). A total of 44,754 gene expressions were recorded among 163
adult patients using Affymetrix HG-U133 A1B microarrays. Based the
method of supervised principal components
\citep{bair04:_semi_super_method_predic_patien}, the gene
expressions were used to develop an 86-gene signature for predicting
survival. The signature was validated on an external test data set
consisting of 79 patients profiled using Affymetrix HG-U133 Plus 2.0
microarrays. All data  is publicly
available on the Gene Expression Omnibus web site
(http://www.ncbi.nlm.nih.gov/geo/) under the accession number
GSE12417. The CN-AML data was recently used by
\cite{benner10:_high_dimen_cox_model} for comparing the
performance of variable selection~methods.

Median survival time was 9.7 months in the training data (censoring
rate 37\%) and 17.7 months in the test data (censoring rate 41\%).
Preliminary to analysis, we followed the scaling approach employed by
\cite{metzeler08} and centered the gene expressions separately within
the test and training data set, followed by a scaling of the training
data with respect to the test
data.  

We first applied vanilla FAST-SIS to the $n=163$ patients in the
training data to reduce the dimension from $p=44,754$ to $d=\lfloor
n/\log(n)\rfloor=31$. We then used OS-SCAD to select a final set
among these 31 genes. Since the PBIC criterion can be somewhat
conservative in practice, we selected the OS-SCAD tuning parameter using 5-fold
cross-validation based on the loss
function~(\ref{eq:lossf-for-ahaz}). Specifically, using a random split
of $\{1,\ldots,163\}$ into folds $F_1,\ldots,F_5$ of approximately equal size, we chose
$\lambda$ as:
\begin{displaymath}
  \hat\lambda = \argmin_\lambda \sum_{i=1}^5 L^{(F_i)}\{\hat{\vekg\beta}_{-F_i}(\lambda)\};
\end{displaymath}
with $L^{(F_i)}$ the loss function using only observations from
$F_i$ and $\hat{\vekg\beta}_{-F_i}(\lambda)$ the regression
coefficients estimated for a tuning parameter $\lambda$, omitting
observations from $F_i$. This approach yielded a set of 7 genes, 5 of
which also appeared in the signature of \cite{metzeler08}. For
$\hat{\vekg{\beta}}$ the estimated penalized regression coefficients, we
calculated a risk score $\vekg{Z}_j^\top\hat{\vekg{\beta}}$
for each patient in the test data.  In a Cox  model, the
standardized risk score had a hazard ratio of 1.69 ($p=6 \cdot
10^{-4}$; Wald test). In comparison, lasso based on the
Lin-Ying model (\cite{leng07}; \cite{martinussen09:_covar_selec_semip}) with
5-fold cross-validation gave a standardized risk score with
a hazard ratio of 1.56 ($p=0.003$; Wald test) in the test data,
requiring 5 genes; \cite{metzeler08} reported a hazard ratio of 1.85
($p = 0.002$) for their 86-gene signature.

We repeated the above calculations for the three scaled versions of
the FAST statistic (\ref{eq:tfast})-(\ref{eq:lossfast}). Since
assessment of prediction performance using only a single data set may
be misleading, we also validated the screening methods via
leave-one-out (LOO) cross-validation based on the 163 patients in the
training data. For each patient $j$, we used FAST-SIS as above (or
Lin-Ying lasso) to obtain regression coefficients
$\hat{\vekg{\beta}}_{-j}$ based on the remaining 162 patients and
defined the $j$th LOO risk score as the percentile of $\vek{Z}_j^\top
\hat{\vekg{\beta}}_{-j}$ among $\{\vek{Z}_i^\top
\hat{\vekg{\beta}}_{-j}\}_{i \neq j}$. We calculated Wald $p$-values
in a Cox regression model including the LOO score as a continuous
predictor. Results are shown in Table~\ref{tab:tab-real-example} while
Table \ref{tab:tab-real-example-overlap} shows the overlap between
gene sets selected in the training data. There is seen to be some overlap
between the different methods, particularly between vanilla FAST-SIS
and the lasso, and many of the selected genes also appear in the
signature of \cite{metzeler08}. In the test data, the prediction
performance of the different screening methods was comparable whereas
the lasso had a slight edge in the LOO
calculations. Lin-Ying SIS selected only a single gene in the test
data and typically selected no genes in the LOO calculations. We
found FAST screening to be slightly more sensitive to the
cross-validation procedure than the lasso.

\begin{table}
  \caption{\label{tab:tab-real-example}Prediction performance of FAST-SIS and Lin-Ying lasso in the AML data, evaluated in terms of the Cox hazard ratio  for the standardized continuous risk score. The LOO calculations are based on the training data only.}
\centering
\fbox{%
{\footnotesize
  \begin{tabular}{llrrrrrr}
\\[-7pt]
&&&\multicolumn{5}{c}{\textsl{Screening method}} \\[5pt]
\cline{4-8}\\[-5pt]
\textsl{Scenario} & \textsl{Summary statistic} && $\vek{d}$ & $\vek{d}^{\mathrm{LY}}$ & $\vek{d}^{|Z|}$ & $\vek{d}^{\mathrm{loss}}$& Lasso \\[5pt]
\hline
\\[-5pt]
Test data   & Hazard ratio&& 1.69 & 1.59 & 1.46 & 1.58 & 1.54  \\ 
& $p$-value&& $6 \cdot 10^{-4}$ & 0.0007 & 0.01 & 0.002  & 0.004\\ 
   & No. predictors && 7 & 1 & 3 & 7 & 5   \\[3pt] 
LOO  & $p$-value&& $4 \cdot 10^{-7}$ & $0.16$ & $5 \cdot 10^{-5}$ & $4 \cdot 10^{-4}$ &   $4 \cdot 10^{-8}$\\ 
   & Median no. predictors && 7 & 0 & 3 & 5  & 5\\[5pt]  
  \end{tabular}
 }}
\end{table}

\begin{table}
 \caption{Overlap between  gene sets selected  by the different screening methods and the signature of \cite{metzeler08}.\label{tab:tab-real-example-overlap}}
\centering
\fbox{%
{\footnotesize
  \begin{tabular}{c| ccccccc}
\\[-7pt]
& $\vek{d}$ & $\vek{d}^{\mathrm{LY}}$ & $\vek{d}^{|Z|}$ & $\vek{d}^{\mathrm{loss}}$& Lasso & Metzeler\\[5pt]
\cline{1-7}\\[-5pt]
 $\vek{d}$            & 7 & 0 & 1 & 2 & 4 & 5\\
$\vek{d}^{\mathrm{LY}}$   &   & 1 & 0 & 0 & 0 &0    \\
$\vek{d}^{|Z|}$         &   &   & 3 & 2 & 2 &2 \\   
$\vek{d}^{\mathrm{loss}}$  &   &   &   & 7 & 2 & 5 \\
Lasso                  &   &   &   &   & 5  & 5\\
Metzeler               &  \phantom{Lasso} &  \phantom{Lasso} & \phantom{Lasso}  & \phantom{Lasso}  &  \phantom{Lasso} & 86 \\[5pt]
  \end{tabular}
 }
}
\end{table}
We next evaluated the extent to which iterated FAST-SIS might improve
upon the above results.  From our limited experience with applying
ISIS to real data,  instability can become an issue
when several iterations of ISIS are run; particularly when
cross-validation is involved. Accordingly, we ran only a single
iteration of ISIS using $Z$-FAST-ISIS. The algorithm kept 2 of the
genes from the first FAST-SIS round and selected 3 additional
genes so that the total number of genes was 5. Calculating in the test
data a standardized risk score based on the final regression coefficients,
we obtained a Cox hazard ratio of only 1.06 ($p=0.6$; Wald test) which
is no improvement over non-iterated FAST-SIS. A similar conclusion was
reached for the corresponding LOO calculations in the training data
which gave a Cox Wald $p$-value of 0.001 for the LOO risk score, using
a median of 4 genes.  None of the other FAST-ISIS methods lead to
improved prediction performance compared to their non-iterated
counterparts.

FAST-ISIS runs swiftly on this large data set: one iteration of
the algorithm (re-recruitment and OSS-SCAD feature selection with
5-fold cross-validation) completes in under 5 seconds on a standard
laptop.

Altogether, the example shows that FAST-SIS can compete with a
computationally more demanding full-scale variable selection method in
the sense of providing similarly sparse models with competitive
prediction properties. FAST-ISIS, while computationally very feasible,
did not seem to improve prediction performance over simple independent
screening in this particular data set.

\section{Discussion}
Independent screening -- the general idea of looking at the effect of
one feature at a time -- is a well-established method for
dimensionality reduction. It constitutes a simple and excellently
scalable approach to analyzing high-dimensional data. The SIS property
introduced by \cite{fan08:_sure} has enabled a basic formal assessment
of the reasonableness of general independent screening
methods. Although the practical relevance of the SIS property has been
subject to scepticism \citep{roberts00:_discus_sure}, the formal context
needed to develop the SIS property is clearly useful for
identifying the many implicit assumptions made when applying
univariate screening methods to multivariate data.

We have introduced a SIS method for survival
data based on the notably simple FAST statistic. In simulation
studies, FAST-SIS performed on par with SIS based on the popular Cox
model, while being considerably more amenable to analysis. We have
shown that FAST-SIS may admit the formal SIS property
within a class of single-index hazard rate models.  In addition to
assumptions on the feature distribution which are well known in the
literature, a principal assumption for the SIS property to hold is
that censoring times do not depend on the relevant features nor
survival. While such partially random censoring may be appropriate
to assume in many clinical settings, it indicates that additional
caution is called for when applying univariate screening and competing
risks are suspected.

A formal consistency property such as the SIS property is but one
aspect of a statistical method and does not make FAST-SIS universally
preferable.  Not only is the SIS property unlikely to be unique to
FAST screening, but different screening methods often highlight
different aspects of data \citep{ma11:_rankin}, making it impossible
and undesirable to recommend one generic method.  We do, however,
consider FAST-SIS a good generic choice of initial screening method
for general survival data. Ultimately, the initial choice of a
statistical method is likely to be made on the basis of parsimony,
computational speed, and ease of implementation.  The FAST statistic
is about as difficult to evaluate as a collection of 
correlation coefficients while iterative FAST-SIS
only requires solving one linear system of  equations. This yields
substantial computational savings over methods not sharing
the advantage of linearity of estimating equations.

Iterated SIS has so far been studied to a very limited extent in an
empirical context. The iterated approach works well on simulated data,
but it is not obvious whether this necessarily translates into good
performance on real data. In our example involving a large gene
expression data set, ISIS did not improve results in terms of
prediction accuracy. Several issues may affect the performance of ISIS
on real data. First, it is our experience that the `Rashomon effect',
the multitude of well-fitting models \citep{breiman01:_statis_model},
can easily lead to stability issues for this type of forward
selection. Second, it is often difficult to choose a good tuning
parameter for the variable selection part of ISIS. Using BIC may lead
to overly conservative results, whereas cross-validation may lead to
overfitting when only the variable selection step -- and not the
recruitment steps -- are cross-validated. \cite{he11} recently
discussed how to combine ISIS with stability selection
\citep{meinshausen10:_stabil} in order to tackle instability issues
and to provide a more informative output than the concise `list of
indices' obtained from standard ISIS.  Their proposed scheme requires
running many subsampling iterations of ISIS, a purpose for which
FAST-ISIS will be ideal because of its computational efficiency. The
idea of incorporating stability considerations is also attractive from
a foundational point of view, being a pragmatic departure from the
limiting \emph{de facto} assumption that there is a single, true
model. Investigation of such computationally intensive frameworks,
alongside a study of the behavior of ISIS on a range of
different real data sets, is a pertinent future research topic.

A number of other extensions of our work may be of interest. We have
focused on the important case of time-fixed features and
right-censored survival times but the FAST statistic can also be used
with time-varying features alongside other censoring and truncation
mechanism supported by the counting process formalism. Theoretical
analysis of such extensions is a relevant future research topic, as is
analysis of more flexible, time-dependent scaling strategies for the
FAST statistic.  \cite{fan11:_nonpar} recently discussed SIS where
features enter in nonparametric, smooth manner, and an extension of
their framework to FAST-SIS appears both theoretically and
computationally feasible. Lastly, the FAST statistic is closely
related to the univariate regression coefficients in the Lin-Ying
model which is rather forgiving towards misspecification: under
feature independence, the univariate estimator is consistent whenever
the particular feature under investigation enters the hazard rate model as a
linear function of regression coefficients \citep{hattori06:_some}. The Cox model
does not have a similar property \citep{struthers86:_missp}. Whether
such internal consistency under misspecification or lack hereof
affects screening in a general setting is an open question.

 \section*{Appendix: proofs}
 \label{sec:appendix}
 In addition to Assumptions 1-4 stated in the main text, we will make use of
 the following assumptions for the quantities defining the class of 
 single-index hazard rate models \eqref{eq:model-for-haz}:
\begin{enumerate}
\item[\textbf{A.}] $\E(Z_{1j})=0$ and $\E(Z_{1j}^2) =1$, $j=1,2,\ldots,p_n$.
\item[\textbf{B.}] $\ssh\{Y_1(\tau)=1\}>0$.
\item[\textbf{C.}] $\var(\vek{Z}_1^\top \vekg{\alpha}^0)$ is uniformly bounded above.
\end{enumerate} 
The details in Assumption A are included mainly for convenience; it suffices to assume that $\E(Z_{1j}^2) <\infty$.

Our first lemma is a basic symmetrization result, included for completeness.
 \begin{lemma}
\label{lem:symmetrize}
Let $X$ be a random variable with mean $\mu$ and finite variance $\sigma^2$. For $t> \sqrt{8}\sigma $, it holds that $\ssh(|X-\mu|>t) \leq 4
\ssh(|X|>t/4)$.
 \end{lemma}
 \begin{proof}
   First note that when $t> \sqrt{8}\sigma $ we have
   $\ssh(|X-\mu|>t/2) \leq 1/2$, by Chebyshev's inequality.  Let $X'$
   be an independent copy of $X$. Then
   \begin{equation}
\label{eq:probeq1}
     2\ssh(|X| \geq t/4) \geq \ssh(|X'-X|>t/2) \geq \ssh(|X-\mu|>t \land |X'-\mu|\leq t/2).
   \end{equation}
   But
   \begin{displaymath}
     \ssh(|X-\mu|>t \land |X'-\mu|\leq t/2)=\ssh(|X-\mu|>t)\ssh(|X'-\mu| \leq t/2) \geq \frac{1}{2}\ssh(|X-\mu|>t).
   \end{displaymath}
   Combining this with \eqref{eq:probeq1}, the statement of the lemma follows.
 \end{proof}
 The next lemma provides a universal exponential bound for the FAST
 statistic and is of independent interest. It bears some similarity to
 exponential bounds reported by
 \cite{bradic10:_regul_coxs_propor_hazar_model} for the Cox~model.
\begin{lemma}
\label{prop:tail-bound}
Under assumptions A-B there exists constants $C_1,C_2,C_3>0$
independent of $n$ such that for any $K>0$ and $1 \leq j \leq
p_n$, it holds that
  \begin{displaymath}
    \ssh\{n^{1/2}|d_{j}-\delta_{j}|>C_1(1+t)\} \leq 10\exp\{-t^2/(2K^2)\}+C_2\exp(-C_3n)+n\ssh(|Z_{1j}|>K).
  \end{displaymath}
\end{lemma}
\begin{proof}
  Fix $j$ throughout. Assume first that $|Z_{ij}| \leq K$ for some finite $K$.  Define the random variables
  \begin{displaymath}
    A_n:=n^{-1}\sum_{i=1}^n\int_0^\tau\{Z_{ij}-e_j(t)\}\d N_i(t),\quad  B_n:=\int_0^\tau\{\bar{Z}_j(t)-e_j(t)\}\d \bar{N}(t); 
  \end{displaymath}
where $\bar{N}(t):=n^{-1}\{N_1(t)+\cdots+N_n(t)\}$ and $e_j(t)=\E\{\bar{Z}_j(t)\}$. Then we can write
  \begin{displaymath}
      n^{1/2}(d_{j}-\delta_{j})  = n^{1/2}\{A_n-\E(A_n)\}+n^{1/2}\{B_n-\E(B_n)\}.
  \end{displaymath}
  We will deal with each term in the display separately. Since $\d N_i(t) \leq 1$, it holds that 
  \begin{displaymath}
    |A_n| \leq \max_{1 \leq i \leq n}|Z_{ij}|+\|e_j\|_\infty
  \leq 2K.
  \end{displaymath}
 and Hoeffding's inequality \citep{hoeffding63:_probab} implies
  \begin{equation}
\label{eq:tail-one}
    \ssh(n^{1/2}|A_n-\E(A_n)|>t) \leq 2\exp\{-t^2/(2K^2)\}.  
  \end{equation}
  Obtaining an analogous bound for $n^{1/2}\{B_n-\E(B_n)\}$ requires a
  more detailed analysis.  Since $\d \bar{N}(t) \leq 1$, 
  \begin{equation}
\label{eq:expression-for-Bn}
    |B_n| \leq \int_0^\tau|\bar{Z}_j(t)-e_j(t)|\d \bar{N}(t) \leq \|\bar{Z}_j-e_j\|_\infty.
  \end{equation}
  We will obtain an exponential bound for the right-hand side via empirical process
  methods.  Define $E^{(k)}(t):=n^{-1}\sum_{i=1}^n Z_{ij}^{k}
  Y_i(t)$ and $e^{(k)}(t):=\E\{E^{(k)}(t)\}$ for $k=0,1$. Denote
  $m:=\inf_{t \in [0,\tau]}e^{(0)}(t)$ and observe that $m>0$, by Assumption B. Moreover, by Cauchy-Schwartz's inequality,
  \begin{displaymath}
\|e^{(1)}/e^{(0)}\|_\infty \leq m^{-1}\sqrt{\E|Z_{1j}|^2e^{(0)}(t)} \leq m^{-1}.
  \end{displaymath}
  Define $\Omega_n:=\{\inf_{t \in [0,\tau]}E^{(0)}(t) \geq m/2\}$ and
  let $\1_{\Omega_n}$ be the indicator of this event. In
  view of the preceding display, we can write
  \begin{align}
    |\bar{Z}_j(t)-e_j(t)|\1_{\Omega_n} 
    & \leq  \frac{1}{E^{(0)}(t)}\Big\{\Big|\frac{e^{(1)}(t)}{e^{(0)}(t)}\Big||e^{(0)}(t)-E^{(0)}(t)|+|E^{(1)}(t)-e^{(1)}(t)|\Big\}\1_{\Omega_n} \\
    & \leq  2m^{-2}(\|\ssh_n -\ssh\|_{\mc{F}_0}+\|\ssh_n -\ssh\|_{\mc{F}_1})\1_{\Omega_n}\label{eq:someempproc}
  \end{align}
  with function classes $\mc{F}_k:= \{t \mapsto Z^{k} \1(T \geq t \land C \geq
  t)\}$. We proceed to establish exponential bounds for the 
  empirical process suprema in \eqref{eq:someempproc}. Each of the
  $\mc{F}_k$s are Vapnik-Cervonenkis subgraph
  classes, and from \cite{pollard89:_asymp} there exists
  some finite constant $\zeta$ depending only on  intrinsic
  properties of the $\mc{F}_k$s such that
  \begin{equation}
\label{eq:emp-proc-momentbound}
    \E(\|\ssh_n -\ssh\|_{\mc{F}_k}^2) \leq \zeta n^{-1} \E (Z_{1j}^2) =n^{-1}\zeta.
  \end{equation}
  In particular, it also holds that $ \E(\|\ssh_n -\ssh\|_{\mc{F}_k})
  \leq n^{-1/2}\zeta^{1/2}$. Moreover,
  \begin{displaymath}
    |Z_{1j}^{k} \1(T_1 \geq t \land C_1 \geq t)-Z_{1j}^{k} \1(T_1 \geq s \land C_1 \geq s)|^2 \leq K^{2k}, \quad s,t \in [0,\tau].
  \end{displaymath}
  With $k_1:=\zeta^{1/2}$,  the concentration theorem of \cite{massart00:_about_talag} implies 
  \begin{equation}
\label{eq:empproc-bound}
    \ssh\{n^{1/2}\|\ssh_n -\ssh\|_{\mc{F}_k} > k_1(1+t)\} \leq \exp\{-t^2/(2K^2)\}, \quad k=0,1.
  \end{equation}
  Combining \eqref{eq:expression-for-Bn}-\eqref{eq:someempproc}, taking $k_2:=k_1 m^2/2$, we obtain
  \begin{equation}
\label{eq:first-prob-bound}
    \ssh (\{n^{1/2} |B_n| > k_2(1+t)\} \cap \Omega_n )  \leq 2\exp\{-t^2/(2K^2)\}.
  \end{equation}
  whereas Cauchy-Schwarz's inequality implies
\begin{displaymath}
  \E (B_n^2\1_{\Omega_n}) \leq \E \|\bar{Z}_j-e_j\|_\infty^2 \1_{\Omega_n}  
\leq 4m^{-4}\E\{(\|\ssh_n -\ssh\|_{\mc{F}_0}+\|\ssh_n -\ssh\|_{\mc{F}_1})^2\}\1_{\Omega_n} 
  \leq 12 m^{-4}\zeta n^{-1}.
\end{displaymath}
Combining Lemma \ref{lem:symmetrize} and
\eqref{eq:first-prob-bound}, there exists nonnegative
$k_3$ (depending only on $m$ and $\zeta$) such that
\begin{equation}
\label{eq:tail-two}
  \ssh \{n^{1/2} |B_n-\E(B_n)|\geq k_3(1+t) \} \leq 8\exp\{-t^2/(2K^2)\}+ \ssh(\Omega_n^c).
\end{equation}

To bound $\ssh(\Omega_n^c)$, recall that $e^{(0)}(t) \geq m$ by assumption. Consequently,
\begin{displaymath}
  \Omega_n^c
  \subseteq \{|E^{(0)}(t)-e^{(0)}(t)|>m/2  \textrm{ for some } t\} \subseteq \{\|\ssh_n
  -\ssh\|_{\mc{F}_0}>m/2\}.
\end{displaymath}
By \eqref{eq:emp-proc-momentbound}, we have $\E(\|\ssh_n
-\ssh\|_{\mc{F}_0}) \leq m/4$ eventually. By another application of the
concentration theorem \citep{massart00:_about_talag}, there exists
finite $k_4$ so that $ \ssh\{\|\ssh_n
-\ssh\|_{\mc{F}_0}>m/4(1+t)\} \leq k_4\exp(-nt^2/2)$. Setting $t=1$,
\begin{displaymath}
\ssh(\Omega_n^c) \leq
\ssh(\|\ssh_n -\ssh\|_{\mc{F}_0}>m/2) \leq
k_4\exp(-n/2).  
\end{displaymath}
Substituting this bound in \eqref{eq:tail-two} and combining with
\eqref{eq:tail-one}, omitting now the assumption that $Z_{ij}$ is bounded,
it follows that there exists constants
$C_1,C_2,C_3>0$ such that for any $K>0$ and $t>0$,
  \begin{displaymath}
          \ssh\{n^{1/2}|d_{j}-\delta_{j}|>C_1(1+t)\}  
           \leq 10\exp\{-t^2/(2K^2)\}+C_2\exp(-C_3 n)+\ssh\Big(\max_{1 \leq i \leq n}|Z_{ij}| > K\Big).
  \end{displaymath}
The  statement of the lemma then follows from the union bound.
\end{proof}
\begin{lemma}
\label{prop:sure-pre-sreen}
Suppose that Assumptions A-B hold and that there exists positive constants
$l_0,l_1,\eta$ such that $\ssh(|Z_{1j}|>s) \leq l_0\exp(-l_1 s^\eta)$
for sufficiently large $s$. If $\kappa<1/2$ then for any $k_1>0$ there
exists $k_2>0$ such that
  \begin{equation}
\label{eq:sure-scr-1}
     \ssh\Big(\max_{1 \leq j \leq p_n}|d_{j}-\delta_{j}|>k_1 n^{-\kappa}\Big) \leq O[p_n\exp\{-k_2n^{(1-2\kappa)\eta/(\eta+2)}\}].
  \end{equation}
  Suppose in addition that $|\delta_{j}|>k_3n^{-\kappa}$ whenever $j
  \in \mc{M}_\delta^n$ and that $\gamma_n=k_4 n^{-\kappa}$
  where $k_3,k_4$ are positive constants and $k_4 \leq k_3/2$. Then
  \begin{equation}
\label{eq:sure-scr-2}
    \ssh(\mc{M}_\delta^n \subseteq \widehat{\mc{M}}_{d}^n) \geq 1-O[p_n\exp\{-k_2n^{(1-2\kappa)\eta/(\eta+2)}\}].
  \end{equation}
  In particular, if $\log p_n=o\{n^{(1-2\kappa)\eta/(\eta+2)}\}$ then $\ssh(\mc{M}_\delta^n \subseteq \widehat{\mc{M}}_{d}^n) \to 1$ when $n \to \infty.$
\end{lemma}
\begin{proof}
  In Lemma \ref{prop:tail-bound}, take $1+t=k_1 n^{1/2-\kappa}/C_1$
  and $K:=n^{(1-2\kappa)/(\eta+2)}$. Then there exists positive constants
  $\tilde{k}_2,\tilde{k}_3$ such that for each
  $j=1,\ldots,p_n$,
  \begin{displaymath}
    \ssh(|d_{j}-\delta_{j}|>k_1 n^{-\kappa}) \leq 10\exp\{-\tilde{k}_2n^{(1-2\kappa)\eta/(\eta+2)}\}+nl_0\exp\{-\tilde{k}_3 n^{(1-2\kappa)\eta/(\eta+2)}\}.
  \end{displaymath}
  By the union bound, there exists $k_2>0$ such that
\begin{displaymath}
  \ssh\Big(\max_{1 \leq j \leq p_n}|d_{j}-\delta_{j}|>k_1 n^{-\kappa}\Big)  \leq O[p_n \exp\{-k_2n^{(1-2\kappa)\eta/(\eta+2)}\}];
\end{displaymath}
which proves \eqref{eq:sure-scr-1}. Concerning \eqref{eq:sure-scr-2},
$k_3n^{-\kappa} -|d_{j}|\leq |\delta_{j}-d_{j}|$ by assumption and so
  \begin{displaymath}
    \ssh\Big(\min_{j \in \mc{M}_\delta^n}|d_{j}| < \gamma_n\Big) \leq \ssh\Big(\max_{j \in \mc{M}_\delta^n}|d_{j}-\delta_{j}| \geq k_4 n^{-\kappa}-\gamma_n\Big) \leq \ssh\Big(\max_{j \in \mc{M}_\delta^n}|d_{j}-\delta_{j}| \geq n^{-\kappa} k_3/2 \Big);
  \end{displaymath}
  where the last inequality follows since we assume $k_4 \leq k_3/2$. Taking $k_1=k_3/2$ in \eqref{eq:sure-scr-1},
  we arrive at the desired conclusion:
\begin{displaymath}
  \ssh(\mc{M}_\delta^n \subseteq \widehat{\mc{M}}^n_d)
  \geq 1 -  \ssh\Big(\min_{j \in \mc{M}_\delta^n}|d_{j}| < \gamma_n\Big) 
   \geq 1-O[p_n \exp\{-k_2n^{(1-2\kappa)\eta/(\eta+2)}\}].
\end{displaymath}
Finally, $\ssh(\mc{M}_\delta^n \subseteq
\widehat{\mc{M}}_d^n) \to 1$ when $n\to \infty$ follows
immediately when $\log p_n=o\{n^{(1-2\kappa)\eta/(\eta+2)}\}$.
\end{proof}
\begin{lemma}
\label{lem:fansong}
Let $\vek{Z} \in \mathbb{R}^p$ be a random vector with zero mean and
covariance matrix $\vekg{\Sigma}$. Let $\vek{b} \in \mathbb{R}^p$ and suppose that
$\E(\vek{Z}|\vek{Z}^\top \vek{b})=\vek{c}\vek{Z}^\top \vek{b}$ for some constant vector $\vek{c} \in
\mathbb{R}^p$.  Assume that $f$ is some real function. Then
\begin{equation}
\label{eq:ellip-one}
  \E\{\vek{Z} f(\vek{Z}^\top \vek{b})\} = \vekg{\Sigma} \vek{b} \frac{\E\{\vek{Z}^\top \vek{b} f(\vek{Z}^\top \vek{b})\}}{\var(\vek{Z}^\top \vek{b})};
\end{equation}
taking $0/0:=0$.  If moreover $f$ is differentiable and strictly monotonic, there exists
$\varepsilon>0$ such that
\begin{equation}
\label{eq:ellip-two}
   \E |\vek{Z} f(\vek{Z}^\top \vek{b})| \geq \vekg{\Sigma}\vek{b} \varepsilon/\var(\vek{Z}^\top \vek{b}).
\end{equation}
 In particular, $\E\{Z_j f(\vek{Z}^\top \vek{b})\}=0$ iff $\cov(Z_j,\vek{Z}^\top \vek{b})=0$.
\end{lemma}
\begin{proof}
  Set $W:=\vek{Z}^\top \vek{b}$.  By standard properties of conditional expectations, it
  holds that 
  \begin{displaymath}
    0=\E\{W (\vek{Z}-\E(\vek{Z}|W))\}= \vekg{\Sigma} \vek{b} - \E\{W\E(\vek{Z}|W)\}=\vekg{\Sigma} \vek{b} -\vek{c} \E(W^2),
  \end{displaymath}
  implying $\E(\vek{Z}|W)=\vekg\Sigma \vek{b} W/\var(W)$. We then
  obtain~\eqref{eq:ellip-one}:
\begin{displaymath}
  \E\{\vek{Z} f(\vek{Z}^\top \vek{b})\} = \E\{\E(\vek{Z}|W)f(W)\}= \vekg\Sigma \vek{b}\E\{W f(W)\}/\var(W).
\end{displaymath}
To show \eqref{eq:ellip-two}, the mean value theorem implies the
existence of some $0<\tilde{W} < W$ such that
\begin{displaymath}
   \E(W f(W)) = \E[W\{f(0)+f'(\tilde{W})W] = \E\{W^2f'(\tilde{W})\}.
\end{displaymath}
Then 
\begin{displaymath}
   \E|W^2f'(\tilde{W})| \geq \E\{|f'(\tilde{W})| W^2\1(W^2 \leq 1)\} \geq \inf_{0 \leq x \leq 1} |f'(x)| \E\{W^2\1(W^2 \leq 1)\}.
\end{displaymath}
Strict monotonicity of $f$ then  yields \eqref{eq:ellip-two}.
\end{proof}
\begin{lemma}
\label{lemma:cum-haz-decom}
Assume that the survival time $T$ has a general, continuous hazard rate function
$\lambda_T(t|Z)$ depending on the random variable $Z \in
\mathbb{R}$ and that the censoring time $C$ is independent of $Z$,
$T$.  Then
  \begin{displaymath}
    \delta = \int_0^\tau \tilde{e}(t) \d F(t) = \E\{\tilde{e}(T \land C \land \tau)\};
  \end{displaymath}
where $F(t) := \ssh(T
  \land C \land \tau \leq t)$ and $\tilde{e}(t):= \E\{Z
  \ssh(T \geq t|Z)\}/\E\{\ssh(T \geq t)\}$.
\end{lemma}
\begin{proof}
  Let $S_T,S_C$ denote the survival functions of $T,C$,
  conditionally on~$Z$. Using the expression
  \eqref{eq:cp-decom} for $\delta$ alongside the assumption of
  random censoring, we obtain
\begin{align}
  \delta &= \E\Big[\int_0^\tau \{Z-e(t)\}Y(t) \lambda_T(t|Z) \d t\Big]\\
  &= \int_0^\tau S_C(t)\E\{ ZS_T(t)  \lambda_T(t|Z)  \}\d t-\int_0^\tau \frac{\E\{ZS_T(t)\}}{\E\{Y(t)\}}S_C(t)\E\{Y(t)\lambda_T(t|Z)\}\d t \label{eq:expansion2} \\
  &= -\int_0^\tau   \frac{\d}{\d t} \tilde{e}(t)\E\{Y(t)\} \d t;
\end{align}
where last equality follows since $S_T'=-\lambda_T S_T$.  Integrating by parts, we obtain the statement of the lemma:
\begin{displaymath}
  \delta = -\int_0^\tau \frac{\d}{\d t} \tilde{e}(t) \E\{Y(t)\} \d t = -\int_0^\infty \frac{\d}{\d t}\tilde{e}(t) \E\{\ssh(T \land C \land \tau \geq t|Z) \} \d t=   \E\{\tilde{e}(T \land C \land \tau)\}.
\end{displaymath}
\end{proof}
\begin{proof}[Proof of Theorem~\ref{thm:mainthm}]
  Set $\tilde{e}_j(t):= \E\{Z_{1j} S_T(t,\vek{Z}_1^\top
  \vekg{\alpha}^0)\}/\E\{S_T(t,\vek{Z}_1^\top \vekg{\alpha}^0)\}$ with
  $S_T(t,\,\cdot\,)=\exp\{-\int_0^t \lambda(s,\cdot)\d
  s\}$. From Assumptions
  \ref{assumption:monotonicity}-\ref{assumption:ran-cens}, Assumption
  C, and Lemma \ref{lem:fansong},  there exists a
  universal positive constant $k_1$ such that
\begin{displaymath}
  |\delta_j|=|\tilde{e}_j(t)| \geq |\E\{Z_{1j}S_T(t,\vek{Z}_1^\top \vekg{\alpha}^0)\}| \geq k_1|\cov(Z_{1j},\vek{Z}_1^\top \vekg{\alpha}^0)|, \quad j \in \mc{M}^n.
\end{displaymath}
 Then $\mc{M}^n \subseteq \mc{M}_\delta^n$.  The sure
  screening property follows from Lemma \ref{prop:sure-pre-sreen} and
  the assumptions.
\end{proof}
\begin{proof}[Proof of Theorem~\ref{thm:bound-on-sel-var}]
  Suppose that 
  \begin{equation}
\label{eq:bound-on-deltasqnorm}
    \|\vekg{\delta}\|^2 =O\{\lambda_\mathrm{max}(\vekg{\Sigma})\}.
  \end{equation}
 For any $\varepsilon>0$, on the set $B_n:=\{\max_{1 \leq j \leq p_n}|d_{j}-\delta_{j}|\leq \varepsilon n^{-\kappa}\}$, it then holds that
  \begin{displaymath}
    |\{1 \leq j \leq p_n\,:\,|d_{j}|>2\varepsilon n^{-\kappa} \}| \leq |\{1 \leq j \leq p_n\,:\, |\delta_{j}|>\varepsilon n^{-\kappa}\}| \leq O\{n^{2\kappa}
  \lambda_\mathrm{max}(\vekg{\Sigma})\}.    
  \end{displaymath}
Taking $k_1=2\epsilon$ in Lemma \ref{prop:sure-pre-sreen}, we have
\begin{displaymath}
  \ssh[|\widehat{\mc{M}}_d^n| \leq O\{n^{2\kappa} \lambda_\mathrm{max}(\vekg\Sigma)\}] \geq
  \ssh[|\{j \,:\,|d_{j}|>k_1 n^{-\kappa} \}|\leq O\{n^{2\kappa} \lambda_\mathrm{max}(\vekg\Sigma)\}]  \geq \ssh(B_n).
\end{displaymath}
By Lemma \ref{prop:sure-pre-sreen}, $\ssh(B_n)=1-O[p_n \exp\{-c_3 n^{(1-2\kappa)\eta/(\eta+2)}\}]$ as claimed.
  So we need only verify~\eqref{eq:bound-on-deltasqnorm}.

  By Lemma \ref{lemma:cum-haz-decom}, there exists a positive constant $c_1$ such that
  $|\delta_{j}| \leq c_1\int_0^\tau |\E\{Z_{1j} S_T(t,\vek{Z}_1^\top
  \vekg{\alpha}^0)\}|\d F(t)$ for $j \in \mc{M}^n$ with $F$ 
  the unconditional distribution function of $T_1 \land C_1 \land
  \tau$.  In contrast, $\delta_j=0$ for $j \notin \mc{M}^n$, by Assumptions
  \ref{assumption:ran-cens}-\ref{assumption:po}. It follows from Jensen's inequality that there exists a positive constant $c_2$
  such that
  \begin{equation}
\label{eq:deltasqnorm}
    \|\vekg{\delta}\|^2 \leq c_2\int_0^\tau \|\E\{\vek{Z}_{1} S_T(t,\vek{Z}_1^\top \vekg{\alpha}^0)\}\|^2 \d F(t). 
  \end{equation}
 Lemma \ref{lem:fansong} implies
  \begin{equation}
    \E\{\vek{Z}_{1} S_T(t,\vek{Z}_1^\top \vekg{\alpha}^0)\}= \frac{\E\{\vek{Z}_{1}^\top \vekg{\alpha}^0 S_T(t,\vek{Z}_1^\top \vekg{\alpha}^0)\}}{\var(\vek{Z}_1^\top \vekg{\alpha}^0)}\vekg{\Sigma} \vekg{\alpha}^0. \label{eq:alt-expr-for-deltansq}
  \end{equation}
  By Cauchy-Schwartz's inequality, using that $\|\vekg{\Sigma}
  \vekg{\alpha}^0\|^2 \leq \|\vekg{\Sigma}^{1/2}\|^2\|\vekg{\Sigma}^{1/2}\vekg{\alpha}^0\|^2 \leq \lambda_\mathrm{max}(\vekg{\Sigma})\|\vekg{\Sigma}^{1/2}\vekg{\alpha}^0\|^2$,  
\begin{displaymath}
  \|\E\{\vek{Z}_{1} S_T(t,\vek{Z}_1^\top \vekg{\alpha}^0)\}\|^2 \leq \|\vekg{\Sigma} \vekg{\alpha}^0\|^2/\var(\vek{Z}_1^\top \vekg{\alpha}^0) \leq \lambda_\mathrm{max}(\vekg\Sigma).
\end{displaymath}
Inserting this in \eqref{eq:deltasqnorm} then yields~the desired
result~\eqref{eq:bound-on-deltasqnorm}. Note that this result does not
rely on the uniform boundedness of $\var(\vek{Z}_1^\top \vekg{\alpha}^0)$ (Assumption C).
\end{proof}
\begin{lemma}
\label{prop:aalen-screen}
Suppose that Assumption A holds and that both the survival time $T_1$ and censoring time $C_1$
follow a nonparametric Aalen model~\eqref{eq:def-of-aalenmodel} with time-varying
parameters $\vekg{\alpha}^0$ and $\vekg{\beta}^0$, respectively.  Suppose
moreover that $\vek{Z}_1 = \vekg{\Sigma}^{1/2} \tilde{\vek{Z}}_1$
where $\tilde{\vek{Z}}_1$ has i.i.d.~components and denote
by $\phi(x):=\E\{\exp(\tilde{Z}_{j1}x)\}$ the moment generating
function of $\tilde{Z}_{j1}$. Then 
  \begin{equation}
\label{eq:aalen-general-sigma}
    \vekg{\delta} = \vekg{\Sigma}^{1/2}\Big[\int_0^\tau \mathrm{diag}\Big\{\frac{\d}{\d x}\frac{\phi'(x)}{\phi(x)}\Big|_{x=-\Gamma_j^0(t)}\Big\} \E\{Y_1(t)\} \vekg{\alpha}^0(t)^\top \d t \Big]\vekg{\Sigma}^{1/2};
  \end{equation}
  where $\vekg{\Gamma}^0(t):=\vekg\Sigma^{1/2} \int_0^t \{\vekg{\alpha}^0(s)+\vekg{\beta}^0(s)\} \d s$. In particular,
  if $\vek{Z}_1 \thicksim \mc{N}(0,\vekg{\Sigma})$ then
\begin{equation}
\label{eq:aalen-normal-sigma}
\vekg{\delta} = \vekg{\Sigma}\Big\{\int_0^\tau \vekg{\alpha}^0(t)\E\{Y_1(t)\}\d t\Big\}.
\end{equation}
\end{lemma}
\begin{proof}
  Let $\Lambda_T$ and $\Lambda_C$ denote the cumulative baseline
  hazard functions associated with $T_1$ and $C_1$. Combining
  \eqref{eq:cp-decom} and \eqref{eq:def-of-aalenmodel}, we get
  \begin{align}
    \vekg{\delta} &= \E \Big\{\int_0^\tau \vek{Z}_1 \vek{Z}_1^\top Y_1(t) \vekg{\alpha}^0(t) \mathrm{d} t\Big\}  - \int_0^\tau \E\{\vek{Z}_1Y_1(t)\}^{\otimes 2} \E\{Y_1(t)\}^{-1} \vekg{\alpha}^0(t) \mathrm{d}t  \\
    &=  \int_0^\tau \vekg{\Sigma}^{1/2}\vek{H}(t) \vekg{\Sigma}^{1/2} \E\{Y_1(t)\} \vekg{\alpha}^0(t) \mathrm{d}t;
  \end{align}
  defining here
  \begin{displaymath}
    \vek{H}(t):=\frac{\E\{Y_1(t)\} \E\{\tilde{\vek{Z}}_1 \tilde{\vek{Z}}_1^\top Y_1(t)\}-\E\{\tilde{\vek{Z}}_1Y_1(t)\}^{\otimes 2}}{\E\{Y_1(t)\}^2}.
  \end{displaymath}
  Since we have
  $Y_1(t)=\exp[-\{\Lambda_T(t)+\Lambda_C(t)+\tilde{\vek{Z}}_1^\top
  \vekg{\Gamma}^0(t)\}]$ conditionally on $\tilde{\vek{Z}}_1$, independence of the components of $\tilde{\vek{Z}}_1$ clearly implies $ [\vek{H}(t)]_{ij} \equiv 0$ for $i
  \neq j$. For $i=j$, factor the conditional at-risk indicator as
  $Y_1(t)=Y_1^{(j)}(t)Y_1^{(-j)}(t)$ where
  $Y_1^{(j)}:=\exp\{-\tilde{Z}_{1j}\Gamma_j^0(t)\}$. Utilizing independence
  again, we get
\begin{displaymath}
  [\vek{H}(t)]_{jj}=\frac{\E\{Y_1^{(j)}(t)\} \E\{\tilde{Z}_{1j}^2 Y_1^{(j)}(t)\}-\E\{Y_1^{(j)}(t)\tilde{Z}_{1j}\}^2}{\E\{Y_1^{(j)}(t)\}^2}=\frac{\d}{\d x}\frac{\phi'(x)}{\phi(x)}\Big|_{x=-\Gamma_j^0(t)}
\end{displaymath}
This proves \eqref{eq:aalen-general-sigma}. To verify
\eqref{eq:aalen-normal-sigma}, simply note that the moment generating
function of a standard Gaussian is $\phi(x)=\exp(x^2/2)$ for which
$\mathrm{d}/\mathrm{d}x\,(\phi'(x) \phi(x)^{-1}) = 1$.
\end{proof}
From \eqref{eq:aalen-general-sigma}, a `simple' description of
$\vekg{\delta}$ (which does not involve factorizing a matrix in terms of
$\vekg{\Sigma}^{1/2}$) is available exactly when features are
Gaussian. Specifically, it holds for some fixed $K>0$ that
\begin{displaymath}
  \frac{\d}{\d x}\frac{\phi'(x)}{\phi(x)}=K,\quad  \textrm{ and }  \phi(0)=1,
\end{displaymath}
iff $\phi(x)=\exp(Kx^2/2)$, the moment generating function of a
centered Gaussian random variable.
\begin{proof}[Proof of Theorem \ref{thm:mainthm-aalen}]
  We apply Lemma \ref{prop:aalen-screen}. Denote by $\vek{v}_j$ the $j$th
  canonical basis vector in $\mathbb{R}^{p_n}$. Integrating by parts in \eqref{eq:aalen-normal-sigma}, we obtain
  \begin{displaymath}
    \delta_{j}= \vek{v}_j^\top \vekg{\Sigma} \int_0^\tau \vekg{\alpha}^0(t) \E\{Y_1(t)\} \d t=\vek{v}_j^\top \vekg\Sigma \int_0^\infty \vekg{\alpha}^0(t) \E\{\ssh(T_1 \land C_1 \land \tau \geq t)\} \d t=  \vek{v}_j^\top \vekg{\Sigma} \E\{\vek{A}^0(T_1\land C_1\land \tau)\}.
  \end{displaymath}
  By the assumptions, $|\vek{v}_j^\top \vekg{\Sigma} \E\{\vek{A}^0(T_1\land C_1 \land
  \tau)\}|\geq c_1n^{-\kappa}$ whenever $j \in \mc{M}^n$.  Thus
  $\mc{M}^n \subseteq \mc{M}_\delta^n$. For Gaussian $Z_{1j}$,
  we have $\ssh(|Z_{1j}|>s) \leq \exp(-s^2/2)$, and the SIS property then follows from~Lemma~\ref{prop:sure-pre-sreen}.
\end{proof}

\begin{proof}[Proof of Theorem \ref{thm:consistent-joint-single-index}]
  Recall that
  \begin{displaymath}
    \vekg{\Delta} = \E\Big[\int_0^\tau \{\vek{Z}_1-\vek{e}(t)\}^{\otimes 2} Y_1(t) \d t\Big].
  \end{displaymath}
Then
  \begin{displaymath}
    \vekg{\Delta} \vekg{\alpha}^0 = \int_0^\tau \frac{\E\{Y_1(t)\} \E\{Y_1(t)\vek{Z}_1 \vek{Z}_1^\top\vekg{\alpha}^0 \}-\E\{ Y_1(t) \vek{Z}_1^\top \vekg{\alpha}^0\}\E\{Y_1(t)\vek{Z}_1 \}}{\E\{Y_1(t)\}} \d t,
  \end{displaymath}
But by Lemma \ref{lem:fansong} and the assumption of random censoring, 
\begin{displaymath}
   \E\{Y_1(t)\vek{Z}_1 \vek{Z}_1^\top\vekg{\alpha}^0 \} =  \vekg{\Sigma} \vekg{\alpha}^0 \frac{\E\{(\vek{Z}_1^\top \vekg{\alpha}^0)^2 Y_1(t)\}}{\var(\vek{Z}_1^\top \vekg{\alpha}^0)}, \quad \textrm{and }  \E\{\vek{Z}_1 Y_1(t)\}=\vekg{\Sigma} \vekg{\alpha}^0\frac{\E\{Y_1(t)\vek{Z}_1^\top \vekg{\alpha}^0\}}{\var(\vek{Z}_1^\top \vekg{\alpha}^0)}.
\end{displaymath}
So we can construct a function $\xi$ such that $ \vekg{\Delta}
\vekg{\alpha}^0 =\vekg{\Sigma} \vekg{\alpha}^0 \int_0^\tau
\xi(\vekg{Z}_1^\top \vekg{\alpha}^0,t)\d t$ where $\int_0^\tau
\xi(\vek{Z}_1^\top \vekg{\alpha}^0,t)\d t \neq 0$, by nonsingularity of $\vekg{\Delta}$. Similarly, using
Lemma~\ref{lemma:cum-haz-decom}, we may construct a function $\zeta$
such that $\vekg{\delta} = \vekg{\Sigma} \vekg{\alpha}^0 \int_0^\tau
\zeta(\vek{Z}_1^\top \vekg{\alpha}^0,t) \d t$.  Taking $
\nu:=\int_0^\tau \zeta(\vek{Z}_1^\top \vekg{\alpha}^0,t)\d
t/\int_0^\tau \xi(t,\vek{Z}_1^\top \vekg{\alpha}) \d t$,
$\vekg{\beta}^0= \nu \vekg{\alpha}^0$ solves $\vekg{\Delta}
\vekg{\beta}^0 = \vekg{\delta}$.
\end{proof}

\bibliography{litt}

\end{document}